\documentclass{arxivtmpl}
\hidelogo

\usepackage{amsmath,amsfonts,bm}

\def\eqref#1{equation~\ref{#1}}

\def\1{\bm{1}}

\DeclareMathAlphabet{\mathsfit}{\encodingdefault}{\sfdefault}{m}{sl}
\SetMathAlphabet{\mathsfit}{bold}{\encodingdefault}{\sfdefault}{bx}{n}

\usepackage{hyperref}
\usepackage{url}
\usepackage{amssymb}
\usepackage{booktabs}   
\usepackage{nicematrix} 
\usepackage{multirow}   
\usepackage{graphicx}   
\usepackage[table]{xcolor} 
\usepackage{caption}
\usepackage[most]{tcolorbox}
\usepackage{subcaption}
\usepackage{tabularx}
\usepackage{threeparttable}
\usepackage{pifont}      
\usepackage{algorithm}
\usepackage{algpseudocode}
\usepackage{wrapfig}
\usepackage{needspace}
\usepackage{amsthm}
\makeatletter
\@ifundefined{proposition}{\newtheorem{proposition}{Proposition}}{}
\@ifundefined{corollary}{\newtheorem{corollary}{Corollary}}{}
\theoremstyle{remark}
\@ifundefined{remark}{\newtheorem{remark}{Remark}}{}
\makeatother

\newcommand{\mflogo}{\raisebox{-0.22\height}{\includegraphics[height=1.05em]{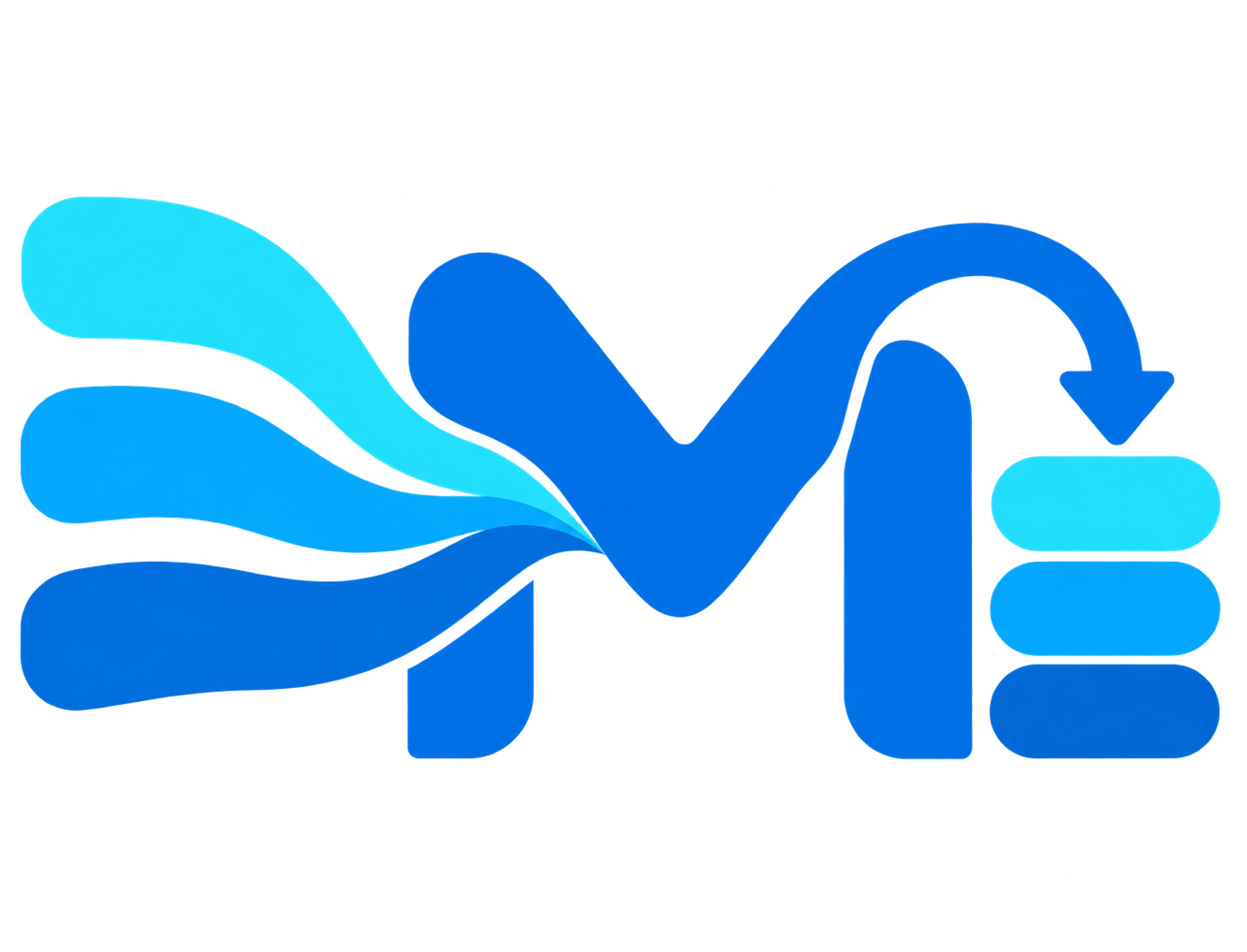}}}
\newcommand{\mflogobig}{\raisebox{-0.26\height}{\includegraphics[height=1.55em]{pic/memflow-logo.png}}}

\usepackage{fontawesome5}
\definecolor{qbg}{HTML}{F4F5FB}
\definecolor{qframe}{HTML}{8A8F98}
\definecolor{qbulb}{HTML}{F2B807}
\newtcolorbox{questionbox}{
  enhanced,
  colback=qbg,
  colframe=qframe,
  boxrule=0.6pt,
  arc=2mm,
  left=2.5mm, right=2.5mm, top=0.9mm, bottom=0.9mm,
  boxsep=0pt,
  before skip=8pt, after skip=8pt,
  fontupper=\small\itshape,
  before upper={\textcolor{qbulb}{\footnotesize\faLightbulb}\hspace{0.5em}}
}

\definecolor{pmA}{HTML}{E91E83}
\definecolor{pmB}{HTML}{4285D4}

\newcommand{\cnum}[1]{\raisebox{-0.4pt}{\ding{#1}}}

\definecolor{casebg}{RGB}{245,245,245}
\definecolor{caseborder}{RGB}{220,220,220}
\definecolor{casegood}{RGB}{24,135,84}
\definecolor{casebad}{RGB}{195,55,65}
\definecolor{casemuted}{RGB}{120,120,120}
\newcommand{\cmark}{\textcolor{casegood}{\ding{51}}}
\newcommand{\xmark}{\textcolor{casebad}{\ding{55}}}

\providecommand{\responseboxheight}{3.4cm}
\newtcolorbox{responsebox}{
  enhanced,
  colback=casebg,
  colframe=caseborder,
  boxrule=0.5pt,
  arc=2.5mm,
  left=2.5mm,
  right=2.5mm,
  top=2.5mm,
  bottom=2.5mm,
  height=\responseboxheight,
  valign=top,
  before upper={\raggedright\setlength{\parindent}{0pt}\setlength{\parskip}{0pt}},
  before skip=0pt,
  after skip=0pt
}

\definecolor{memoryblue}{RGB}{70,105,170}
\definecolor{memorybg}{RGB}{248,250,255}
\definecolor{memorytitle}{RGB}{238,243,253}

\definecolor{memoryblue}{RGB}{70,105,170}
\definecolor{memorybg}{RGB}{248,250,255}
\definecolor{memorytitle}{RGB}{238,243,253}

\newtcolorbox{memorybox}[1]{
  enhanced,
  width=\linewidth,
  colback=memorybg,
  colframe=memoryblue!45,
  colbacktitle=memorytitle,
  coltitle=memoryblue!80!black,
  fonttitle=\bfseries,
  title=#1,
  boxrule=0.7pt,
  arc=3mm,
  left=2mm,
  right=2mm,
  top=1.5mm,
  bottom=1.5mm,
  before skip=6pt,
  after skip=6pt
}

\title{\mflogobig\,\textbf{\texttt{MemFold}}: Learning Compact Soft Memory for Long-Context Personalization via On-Policy Optimization}

\newcommand{\comp}{\mathcal{C}_\phi}
\renewcommand{\responseboxheight}{4.6cm}
\renewcommand\Affilfont{\sffamily\mdseries}
\author[2,*]{Jingxuan Wu}
\author[1,*]{Yuzhe Yang}
\author[3]{Yiqiao Huang}
\author[1]{Chengzhi Liu}
\author[1]{Qingni Wang}
\author[1]{Chengxuan Qian}
\author[4]{Shutong Wu}
\author[4]{Jiawei Zhang}
\author[1]{Xin Eric Wang}
\affil[1]{University of California, Santa Barbara}
\affil[2]{University of North Carolina at Chapel Hill}
\makeatletter
\def\AB@affilsep{\protect\\\protect\Affilfont}
\makeatother
\affil[3]{Harvard University}
\affil[4]{University of Wisconsin--Madison}
\authornote{\textbf{Correspondence:} \href{mailto:jingxwu@unc.edu}{jingxwu@unc.edu}, \href{mailto:yuzheyang@ucsb.edu}{yuzheyang@ucsb.edu}}

\newcommand{\projectlinks}{%
  \href{https://github.com/Johnny221B/memfold}{%
    \adjustbox{valign=c}{\resizebox{!}{12pt}{\textcolor{black}{\faGithub}}}\enspace
    \adjustbox{valign=c}{\textcolor{linkcol}{MemFold}}}%
  \quad
  \href{https://huggingface.co/Johnny221B/memfold}{%
    \adjustbox{valign=c}{\includegraphics[height=12pt]{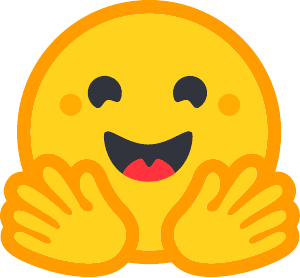}}\enspace
    \adjustbox{valign=c}{\textcolor{linkcol}{MemFold}}}%
  \quad
  \href{https://MemFold.github.io}{%
    \adjustbox{valign=c}{\resizebox{!}{12pt}{\textcolor{black}{\faGlobe}}}\enspace
    \adjustbox{valign=c}{\textcolor{linkcol}{MemFold.github.io}}}%
}

\newtcolorbox{abstractbox}{
  enhanced, breakable, frame hidden, width=\linewidth,
  colback=gblue!6, sharp corners, boxsep=0pt, fontupper=\small,
  left=12pt, right=10pt, top=8pt, bottom=8pt,
  before skip=0pt, after skip=0pt
}

\begin{document}
\begingroup
\setlength{\parskip}{0pt}
\maketitle
\footnotetext[1]{Equal contribution.}
\thispagestyle{firstpagestyle}
\enlargethispage{4\baselineskip}

\nointerlineskip
\vskip\frontmattergap
\begin{abstractbox}
{\sffamily\fontseries{eb}\large\selectfont Abstract}\par

An assistant that serves the same user over a long horizon has to answer from
what that user has revealed: which preferences still hold, which were revised,
and which constraints apply now. Retaining that information is not the same as
acting on it, and the two are usually optimized as if they were. Keeping the
information as text makes the reader's input grow with the retained history,
while compressing it into a fixed number of latent vectors bounds the interface
but is typically trained to reconstruct text or imitate reference answers,
both of which are scored on sequences the reader never produced. We present
\textbf{\texttt{MemFold}}, which optimizes a fixed-budget soft memory by the behavior it
supports. A query-conditioned textual memory is compressed into $K$ continuous
vectors that form the reader's memory interface, and the reader is then trained
on its own rollouts under two complementary signals: group-relative rewards for
task outcomes, and confidence-gated on-policy distillation in which a frozen
textual-memory teacher re-scores the student's sampled tokens under the textual
memory. The teacher is never sampled from, so supervision stays on the student's
current distribution and adds no autoregressive decoding; at inference it is
removed entirely. Across three Qwen backbones, \textbf{\texttt{MemFold}} attains the
highest accuracy we measure on PersonaMem-32K and PersonaMem-128K, with margins
that widen at the longer history length, and transfers to PrefEval and
LongMemEval without target-domain training. Ablations attribute most of the task
gain to the reward term and a smaller additional gain to the teacher signal, and
memory interventions show that the reader depends on the instance-specific
content of its soft memory.

\par\medskip
{\centering\normalsize\Affilfont\projectlinks\par}
\end{abstractbox}
\endgroup

\begin{center}
    \begin{minipage}{\linewidth}
        \centering
        \vspace{-1.5em}
        \includegraphics[width=\linewidth]{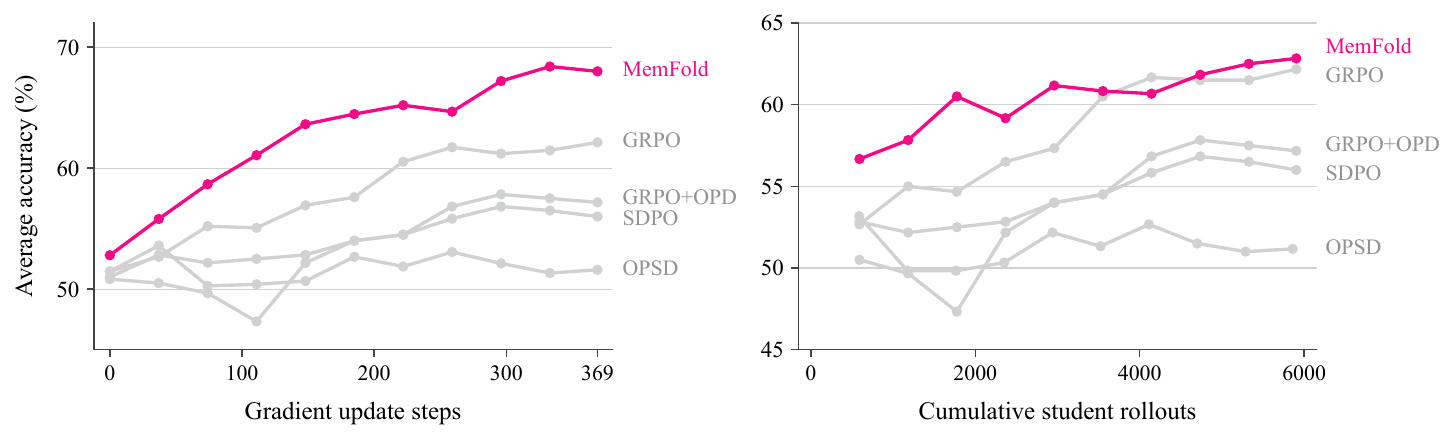}
        \vspace{-1.5em}
        \captionof{figure}{Using Qwen2.5-3B-Instruct on PersonaMem-32K, \textbf{\texttt{MemFold}} improves test accuracy more rapidly during training and achieves higher final performance than the baselines (left), while reaching comparable test accuracy with fewer student rollouts (right), demonstrating both effective optimization and greater sample efficiency. Experimental settings are provided in Appendix~\ref{app:efficiency_comparison}.}
        \label{fig:exp}
    \end{minipage}
\end{center}

\clearpage
\section{Introduction}

A long-running assistant is asked similar questions by different people, and the
suitable answer is not the same for each of them. What makes a response suitable
is the set of facts, preferences, and constraints that this particular user has
revealed in earlier interactions. The same holds for one user across time: a
request phrased identically in March and in September can call for different
advice because what the user wants has changed in between. Personalization in
this setting is therefore a property of the response rather than of the store.
It is the extent to which an answer reflects the information about this user
that is in force at the moment the answer is produced.

That last qualification does much of the work. A long history mixes
enduring preferences with one-off activities, constraints that held only for a
period, explicit revisions, and the reasons behind those revisions. Recency does
not resolve them: the most recent mention of an activity is not necessarily a new
preference, and an early statement is not necessarily stale. On benchmarks built
around evolving user profiles, models recover the course of a preference change
and the reasons behind it more reliably than they let the resulting preference
constrain a concrete recommendation \citep{jiang2025know}, and they continue to
recommend options a user has objected to when the objection was expressed
implicitly rather than as an instruction
\citep{zhao2025llmsrecognizepreferencesevaluating, guo2026realpref}. Evaluations
that score forgetting alongside recall find agents still acting on information a
later turn has invalidated \citep{uddin2026memora}, and condensed histories can
retain the facts of an episode while losing the preference signal that made it
matter \citep{wang2026finperma}. Retaining the evidence and acting on it are
thus separate requirements, and a memory mechanism has to be judged against
both.

Two representations carry that evidence forward. Textual memory keeps it
as text, which is readable and editable, but its length varies with the amount
retained and it is re-encoded at every turn \citep{packer2023memgpt,
zhong2024memorybank, chhikara2025mem0, xu2025amem, luo2026memsif,
liao2026leanmem}. Latent memory replaces the text with a fixed number of
continuous vectors, giving the reader an interface whose size is independent of
how long the interaction has been \citep{mu2023gist, chevalier2023adapting,
ge2024icae, cheng2024xrag, li2026latentcompile, zhou2026latentpress}. That fixed
budget is a property of the interface, and it does not by itself determine
whether the compressed representation still supports the behavior the text did.

The gap this leaves is one of optimization rather than representation.
Compressors are typically trained to reconstruct the source text
\citep{ge2024icae, zhang2026nextmem} or to align with reference answers under a
frozen decoder \citep{rahman2026cmc}, and both objectives score the model on
sequences it did not produce, so neither reaches the errors the reader makes
once the text is gone. Recent work closes part of this loop by treating memory
as a resource a policy learns to use, optimizing what a memory stores,
retrieves, or spends its budget on against a task reward \citep{fu2026latentmem,
feng2026elasticmem, yu2026agemem, yu2026lazymem, ye2026coevomem}. A scalar
outcome, however, reports only that a response was wrong; it does not indicate
which decisions in it failed to use what the memory held. A
characteristic failure of this kind is a response that is fluent and topically
appropriate, with the preference violation confined to one recommendation among
several, which we examine in Section~\ref{sec:discussion}. A sequence-level
reward charges the entire response for that span.

Our starting point is to judge a compact memory not by whether it returns
the text, but by where the reader, given only the compact memory,
under-uses what the text would support. This can be measured directly,
because the textual and compressed memories can be evaluated on the same
response. Given a response sampled by a reader conditioned on soft memory,
a frozen copy of that reader conditioned on the corresponding textual
memory can score the identical token sequence. The resulting token-level
log-probability differences locate where the reader is less confident
under compression than under the text, on trajectories the reader actually
visits. We use them to weight the reader's own tokens rather than as a
target to imitate, so the signal complements the task reward: the reward
indicates whether a response is correct, while the comparison indicates
which of its tokens the textual reading supports more strongly.

Based on this observation, our contributions are threefold:
\cnum{182}~Prior compressors are trained on sequences the reader never produces; we propose \textbf{\texttt{MemFold}}, which optimizes a fixed-budget soft-memory interface directly by the downstream behavior it supports, on the reader's own rollouts.
\ding{183}~We introduce a confidence-gated on-policy distillation objective
that re-scores the student's own generations under textual memory; in
expectation it acts as a reverse KL toward the textual-memory reader with
bounded per-token influence, adding a dense signal to the sequence-level
reward without letting the teacher dominate it.
\cnum{184}~We show that \textbf{\texttt{MemFold}} achieves the highest accuracy on all four benchmarks across three backbones, with margins that widen at longer history lengths, and transfers to unseen benchmarks without target-domain training.

\section{Related Work}

\textbf{Memory and personalization in long-horizon interaction.}
Long-horizon assistants are commonly given an explicit textual store, written
across sessions and queried at inference \citep{lewis2020retrieval,
packer2023memgpt, zhong2024memorybank, xu2025amem, chhikara2025mem0,
fang2026memp, luo2026memsif, liao2026leanmem}, with benchmarks measuring the
recall and multi-session reasoning that results \citep{maharana2024evaluating,
wu2024longmemeval, chang2026locomoconv}. A second line asks whether retained
information changes what the model says to a given user, through personalized
generation from user profiles \citep{salemi2024lamp, salemi2025lamp,
zhang2026memorycd, in2026permembench}, histories in which preferences develop
over time \citep{jiang2025know}, preferences expressed implicitly rather than as
instructions \citep{zhao2025llmsrecognizepreferencesevaluating,
guo2026realpref}, and whether information a later turn has invalidated is
dropped as well as recalled \citep{uddin2026memora}. We adopt the distinction
those benchmarks draw: retrieval accuracy over a history and
preference-consistent behavior in a response are different quantities. Because
the content stays in text, the tokens the reader processes also grow with the
amount retained, and the store is optimized separately from the model consuming
it.

\textbf{Compact memory and behavioral optimization.}
A complementary line shortens the representation, by dropping or rewriting
tokens \citep{jiang2023llmlingua, li2024promptcompression} or, building on prompt
and prefix tuning \citep{lester2021power, li2021prefix}, by encoding context into
continuous vectors read in the embedding space \citep{mu2023gist,
chevalier2023adapting, ge2024icae, cheng2024xrag, zhang2026memgen,
zhang2026nextmem, li2026latentcompile, zhou2026latentpress}, often through
latent-query resamplers \citep{jaegle2022perceiver, alayrac2022flamingo} as ours
is. Their objectives, reconstruction or supervised imitation, are scored on
sequences the reader did not generate, as is distillation that aligns a
compressed context with reference answers under a frozen decoder
\citep{rahman2026cmc}. Reinforcement learning instead trains on
the model's own samples \citep{ouyang2022instructgpt, shao2024deepseekmath,
guo2025deepseekr1} and has been used to decide what a memory stores, retrieves,
or spends its budget on \citep{yan2025memoryr1, wang2025mema, yu2026agemem,
yu2026lazymem, ye2026coevomem, song2026realm, fu2026latentmem,
feng2026elasticmem}, whereas we hold the interface fixed and optimize how it is read;
on-policy distillation adds the per-token signal a scalar reward lacks
\citep{agarwal2024gkd,gu2024minillm,zhao2026self} but uses a stronger
teacher and a signed gap that pulls the student toward it. Ours is closer
to context distillation \citep{snell2022learning}: a frozen copy of the
student's own initialization, never sampled from, re-scores the student's
tokens under the textual memory it has had compressed away, and a
non-negative gate replaces the signed gap. The two differ only in their
memory, which ties the signal to compression rather than to a more
accurate solver.

\section{Methodology}
\label{sec:method}

\textbf{\texttt{MemFold}} has two parts: a fixed-budget memory interface and an
on-policy procedure that optimizes how the reader uses it. A memory writer turns
the history visible at query time into query-relevant textual memory, a
compressor maps that memory into $K$ soft vectors, and the reader is then
trained on its own rollouts under a task reward together with a frozen
textual-memory teacher that re-scores those same rollouts
(Figure~\ref{fig:method}). Supervised training is used only to initialize the
interface; its data construction and training details are deferred to
Appendix~\ref{app:memory_extraction} and Appendix~\ref{app:stage2}.

\begin{figure}[h]

    \centering
    \includegraphics[width=1\linewidth]{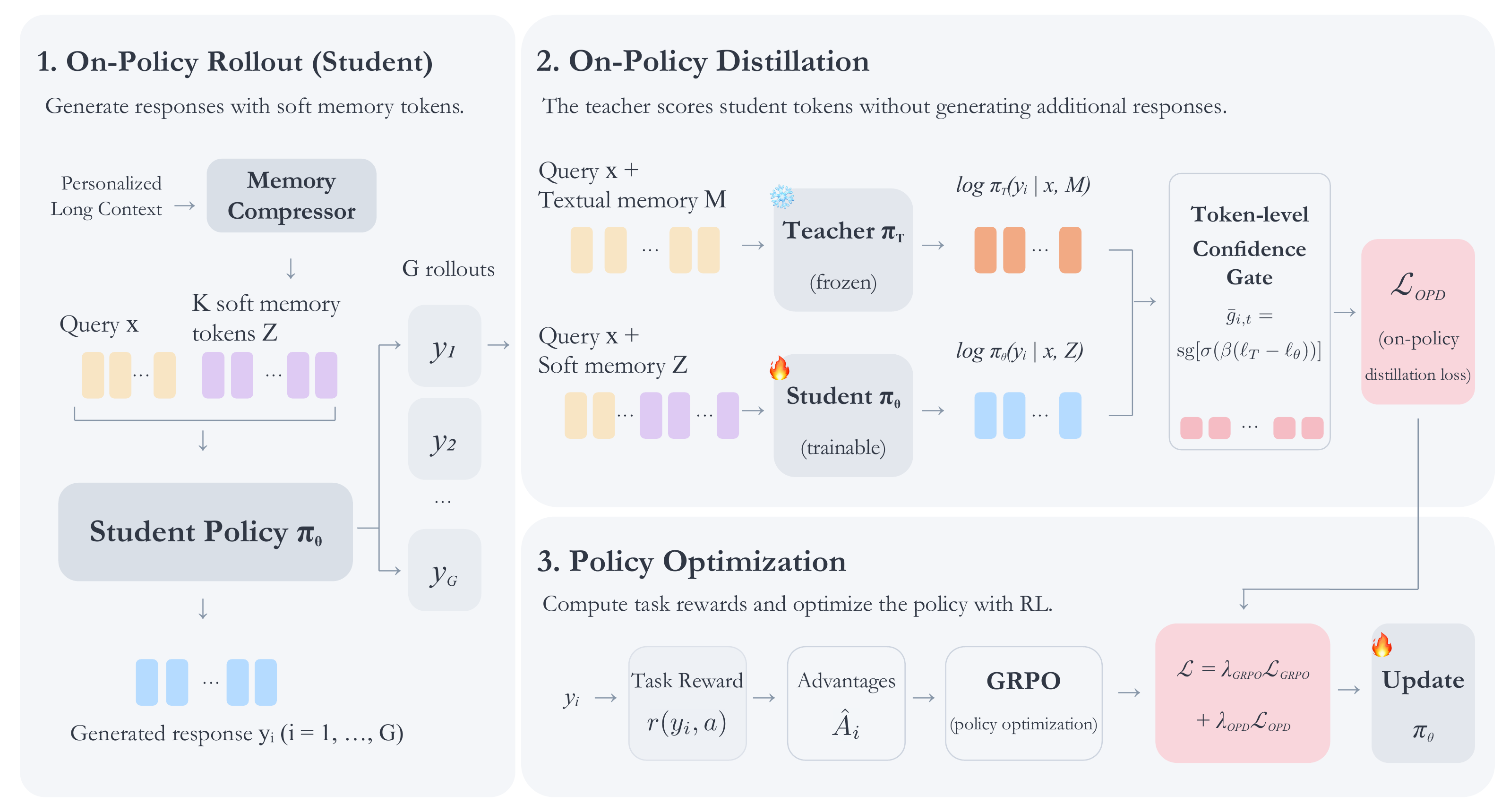}
    \caption{\ding{182} \textbf{On-policy rollout:} The student generates
    responses conditioned on the query and $K$ compressed soft memory tokens.
    \ding{183} \textbf{On-policy distillation:} A frozen textual-memory teacher
    scores the student's tokens, and a detached teacher--student confidence
    gate $\bar g_{i,t}$ weights the student's own tokens.
    \ding{184} \textbf{Policy optimization:} GRPO task rewards and the gated
    distillation term jointly update the student, while the teacher and
    compressor remain frozen.}
    \label{fig:method}
\end{figure}

\subsection{Problem Formulation}
\label{sec:formulation}

A query $x$ arrives at time $\tau$, and the model may condition only on the interaction history $C$ visible before $\tau$. Because no annotation indicates which statements in $C$ remain valid, the model must infer them from the history. All components share a frozen backbone $\theta_0$ and a single LoRA adapter $\theta$. The adapter first produces a query-conditioned textual memory $M=e_{\theta}(C,x)$ containing the relevant evidence, temporal relations, and derived facts. A compressor $\comp$ then maps it to a fixed-size continuous memory
$Z=\comp(E(M))\in\mathbb{R}^{K\times d}$, where $E$ comprises the first four Transformer blocks of $\theta_0$, $K$ is the memory budget, and $d$ is the reader's input embedding dimension.

The same adapter, acting as the reader, generates $y\sim\pi_\theta(\cdot\mid x,Z)$. During on-policy training, a frozen copy of the initialized reader,
$\pi_T(\cdot\mid x,M):=\pi_{\theta_{\mathrm{init}}}(\cdot\mid x,M)$, instead reads the textual memory and serves as a behavioral reference. Teacher and student thus initially share all weights and differ only in their memory inputs. Both score each student-sampled trajectory, and their confidence gap identifies tokens for which the compressed-memory reader is less confident. We use this gap to weight the student's own tokens alongside the task reward,
bounding the teacher's influence on each token rather than matching its
distribution. At inference, only $\pi_\theta(\cdot\mid x,Z)$ is retained, so the memory interface remains fixed at $K$ vectors regardless of the history length.

\subsection{Fixed-Budget Soft Memory}
This component supplies the interface that Section~\ref{sec:on_policy_optimization} acts
on. Two properties are what the method requires of it: its size does not
depend on the length of $C$, and the reader can read it. We instantiate
these requirements with a standard Perceiver-style compressor.

\textbf{Textual memory.} We first transform the long context $C$ into a
structured, query-relevant textual memory $M$, filtering irrelevant history
before compression. The adapter is first trained on memories extracted by
an external model from the history and query alone, and thereafter writes
$M$ itself, so the external model is not needed at inference.

\textbf{Fixed-budget compression.} We encode $M$ with the frozen encoder
$E$, the first four Transformer blocks of the backbone, and a lightweight Perceiver-style compressor
aggregates the variable-length sequence with $K$ learned latent queries $Q_K$:
\begin{equation*}
Z = \mathcal{C}_\phi\big(E(M)\big)
  = P_\phi\big(\mathrm{Comp}_\phi(Q_K, E(M))\big) \in \mathbb{R}^{K\times d},
\end{equation*}
where $P_\phi$ projects the compressed states into the reader's input
embedding space. The reader therefore always receives exactly $K$ memory
vectors, regardless of the lengths of $C$ and $M$.

\textbf{Initialization.} We first pretrain the compressor on raw history
prefixes $h$: the frozen backbone must reconstruct the textual memory from
$Z_h = \mathcal{C}_\phi(E(h))$,
\begin{equation}
\mathcal{L}_{\mathrm{rec}} = -\frac{1}{|M|}\sum_{t=1}^{|M|}
\log p_{\theta_0}\!\left(m_t \mid x, Z_h, m_{<t}\right),
\label{eq:rec}
\end{equation}
then adapt it to textual-memory inputs, $Z = \mathcal{C}_\phi(E(M))$, and initialize the reader on answer generation from $(x, Z)$. These objectives are scored on sequences the reader did not produce, so they make the interface readable without directly optimizing how the reader uses it on its own generations. On-policy optimization addresses this remaining gap.

\begin{algorithm}[h]
\caption{\textbf{\texttt{MemFold}} on-policy memory optimization}
\label{alg:memflow}
\begin{algorithmic}[1]
\Require Reader-initialized adapter $\theta_{\mathrm{init}}$,
compressor $C_\phi$, frozen teacher $\pi_T$; training examples
$(C,x,a)$ with reference answer $a$ and cached on-policy memories
$M^{\mathrm{init}}=e_{\theta_{\mathrm{init}}}(C,x)$

\Ensure Trained adapter $\theta$

\State Initialize $\theta \leftarrow \theta_{\mathrm{init}}$;
freeze compressor, projector, and teacher

\For{each training minibatch of $(C,x,a)$}
    \State Load cached $M^{\mathrm{init}}$ and compute
    $Z \leftarrow C_\phi(E(M^{\mathrm{init}}))$
    \State Snapshot rollout policy
    $\pi_{\mathrm{old}}\leftarrow\operatorname{sg}[\pi_\theta]$
    \State Sample
    $y_1,\ldots,y_G\sim\pi_{\mathrm{old}}(\cdot\mid x,Z)$
    \State Score rewards $r_i=r(y_i,a)$ and compute $\hat A_i$
    \For{each sampled response $i$ and unmasked token $t$}
        \State $\ell_{\theta,i,t}\leftarrow
        \log\pi_\theta(y_{i,t}\mid x,Z,y_{i,<t})$
        \State $\ell_{T,i,t}\leftarrow
        \log\pi_T(y_{i,t}\mid x,M^{\mathrm{init}},y_{i,<t})$
        \State $\bar g_{i,t}\leftarrow
        \operatorname{sg}\!\left[
        \sigma\!\left(\beta(\ell_{T,i,t}-\ell_{\theta,i,t})\right)
        \right]$
    \EndFor
    \State Compute $\mathcal{L}_{\mathrm{GRPO}}$ and
    $\mathcal{L}_{\mathrm{OPD}}$ over shared response masks
    \State Update the LoRA adapter $\theta$
\EndFor

\State \Return $\theta$; at inference, write
$M=e_\theta(C,x)$ and answer from
$(x,C_\phi(E(M)))$ without the teacher
\end{algorithmic}
\end{algorithm}

\subsection{On-Policy Memory Optimization}
\label{sec:on_policy_optimization}

Initialization leaves two gaps. First, a reader trained on reference answers
receives no signal about errors in its own generations. Second, a task reward
does score the reader's own samples, but one scalar per response does not
identify which tokens failed to use what the textual memory held.
\textbf{\texttt{MemFold}} closes the first by optimizing on the student's
rollouts, and the second by re-scoring those same rollouts under the textual
memory. Algorithm~\ref{alg:memflow} summarizes the resulting loop.

\textbf{On-policy rollouts.}
For each query $x$, the current soft-memory policy samples a group of responses
$\{y_i\}_{i=1}^{G}$:
\(
y_i\sim\pi_{\theta}(\cdot\mid x,Z).
\)
Each response receives a task reward $r_i$, from which we compute the
group-relative advantage
\(
\hat{A}_i=\frac{r_i-\mu_r}{\sigma_r+\epsilon}.
\)
Let
\(
\rho_{i,t}(\theta)=
\frac{
\pi_{\theta}(y_{i,t}\mid x,Z,y_{i,<t})
}{
\pi_{\mathrm{old}}(y_{i,t}\mid x,Z,y_{i,<t})
}.
\)
The group-relative policy objective is
\begin{equation}
\mathcal{L}_{\mathrm{GRPO}}
=
-\mathbb{E}\!\left[
\frac{1}{G}\sum_{i=1}^{G}\frac{1}{|y_i|}
\sum_{t=1}^{|y_i|}
\min\!\left(
\rho_{i,t}\hat{A}_i,\,
\operatorname{clip}(\rho_{i,t},1-\epsilon,1+\epsilon)\hat{A}_i
\right)
\right].
\label{eq:grpo}
\end{equation}

\textbf{On-policy memory distillation.} GRPO evaluates complete
responses but does not reveal where the soft-memory reader under-uses
what the memory contains. The textual-memory teacher $\pi_T$ offers a
reference for how the same backbone reads the uncompressed memory: it
scores each token sampled by the student,
\begin{equation}
\ell_{T,i,t} = \log \pi_T(y_{i,t}\mid x, M, y_{i,<t}), \qquad
\ell_{\theta,i,t} = \log \pi_\theta(y_{i,t}\mid x, Z, y_{i,<t}),
\end{equation}
and we weight the student's own tokens by a detached gate
$\bar g_{i,t} = \mathrm{sg}\!\left[\sigma\!\left(\beta(\ell_{T,i,t}-\ell_{\theta,i,t})\right)\right]$:
\begin{equation}
\mathcal{L}_{\mathrm{OPD}} = -\,\mathbb{E}\Big[\frac{1}{G}\sum_{i=1}^{G}
\frac{1}{|y_i|}\sum_{t=1}^{|y_i|} \bar g_{i,t}\,\ell_{\theta,i,t}\Big].
\label{eq:opd}
\end{equation}
The objective strengthens the student's sampled tokens in proportion to
how much more the textual-memory reading supports them than the
soft-memory reading does. Per sample, the gate is largest where the teacher is more confident and approaches zero where it is less confident than the student. The objective never promotes tokens the student did not sample, so it re-ranks the student's
own candidates by the teacher's relative confidence rather than importing the
teacher's own choices; in expectation, this re-ranking acts as a reverse KL
with a bounded per-token coefficient. The bound is
deliberate: $\pi_T$ is a frozen reader of the textual memory, not an accuracy
oracle, so no single token on which it is strongly over- or under-confident can
dominate the update. Signed feedback on task outcomes comes from GRPO.

\textbf{Joint objective and training.}
The two signals address different gaps and are combined as
\[
\mathcal{L}_{\mathrm{MemFold}}
= \lambda_{\mathrm{GRPO}}\mathcal{L}_{\mathrm{GRPO}}
+ \lambda_{\mathrm{OPD}}\mathcal{L}_{\mathrm{OPD}},
\]
with the textual-memory teacher and compressor frozen throughout and no
additional KL regularization in the final configuration.
Because the gate and teacher log-probability are detached, $\mathcal{L}_{\mathrm{OPD}}$
reduces to a gate-weighted likelihood on the student's own samples with no
gradient through the teacher branch; its expected update vanishes when the soft
and textual readings agree and weakens as they converge
(Appendix~\ref{app:theory}).
At inference time, the teacher and all distillation computations are removed,
and generation conditions only on $(x, Z)$; additional implementation details
are in Appendix~\ref{app:stage3}.

\section{Experiments}

\subsection{Experimental Setup}

\textbf{Datasets \& Benchmarks.}
We evaluate on several benchmarks with complementary focuses.
PersonaMem-32K and PersonaMem-128K \citep{jiang2025know} evaluate dynamic user preference tracking under increasingly long conversational histories.
We train on PersonaMem-32K and evaluate directly on the implicit persona subset of PrefEval \citep{zhao2025llmsrecognizepreferencesevaluating} to test cross-dataset personalization generalization.
We also train on LoCoMo \citep{maharana2024evaluating} and evaluate directly on LongMemEval \citep{wu2024longmemeval} without task-specific training to assess generalization to broader long-term memory reasoning tasks.

\textbf{Backbone Models.}
We evaluate our method with three backbone models:
Qwen2.5-3B-Instruct and Qwen2.5-7B-Instruct~\citep{qwen2025qwen25technicalreport},
as well as Qwen3-4B~\citep{yang2025qwen3}.

\textbf{Baselines.}
We organize the baselines into three categories.
For training objectives, we compare with
GRPO~\citep{shao2024deepseekmath} and OPSD~\citep{zhao2026self}.
For latent-memory and context-compression methods, we include
AutoCompressor~\citep{chevalier2023adapting},
MemGen~\citep{zhang2026memgen}, and
xRAG~\citep{cheng2024xrag}. These methods compress textual context into
compact continuous representations for downstream inference.
Finally, Full Text serves as the uncompressed-context baseline, directly
providing the complete textual history to the backbone model.

\textbf{Evaluation.}

We report accuracy (Acc.) and average end-to-end token-equivalent count (\#Tok.) per test instance; accounting rules are in Appendix~\ref{app:token_accounting}. PrefEval uses PersonaMem-32K checkpoints; predictions on PrefEval and LongMemEval are scored via the Gemini-3.8-flash API~\citep{gemini38flash2026} with low thinking, with LongMemEval additionally using structured yes/no outputs. Greedy decoding ($T=0$) is used by default; ablations report mean accuracy and Pass@16 over 16 samples ($T=1.0$, top-$p=0.98$).


\begin{table*}[!ht]
\centering
\definecolor{tableblue}{HTML}{4285F4}
\definecolor{tableheader}{HTML}{F2F2F2}
\colorlet{tablebest}{tableblue!55}
\colorlet{tablesecond}{tableblue!22}
\caption{Main results for in-domain and direct cross-dataset evaluation,
reporting accuracy and end-to-end token-equivalent counts.
Weighted averages use benchmark sample counts. {\setlength{\fboxsep}{0pt}\protect\colorbox{tablebest}{\kern2pt Best\kern2pt}} and {\setlength{\fboxsep}{0pt}\protect\colorbox{tablesecond}{\kern2pt second-best\kern2pt}} accuracies are highlighted for each backbone and benchmark.}

\label{tab:main_tab}
\small
\setlength{\tabcolsep}{3.5pt}
\renewcommand{\arraystretch}{1.16}
\setlength{\aboverulesep}{0pt}
\setlength{\belowrulesep}{0pt}

\resizebox{\textwidth}{!}{%
\begin{NiceTabular}{llrr rr rr rr rr}
\CodeBefore
\tikz \fill [tableheader] (1-|1) rectangle (4-|13);
\Body
\toprule
\multirow{3}{*}{\textbf{Model}} &
\multirow{3}{*}{\textbf{Method}} &
\multicolumn{4}{c}{\textbf{In-Domain}} &
\multicolumn{4}{c}{\textbf{Out-of-Domain}} &
\multicolumn{2}{c}{} \\
\cmidrule(lr){3-6}\cmidrule(lr){7-10}
& &
\multicolumn{2}{c}{\textbf{PersonaMem-32K}} &
\multicolumn{2}{c}{\textbf{PersonaMem-128K}} &
\multicolumn{2}{c}{\textbf{PrefEval}} &
\multicolumn{2}{c}{\textbf{LongMemEval}} &
\multicolumn{2}{c}{\textbf{Weighted Avg.}} \\
\cmidrule(lr){3-4}\cmidrule(lr){5-6}
\cmidrule(lr){7-8}\cmidrule(lr){9-10}
& &
Acc.$\uparrow$ & \#Tok.$\downarrow$ &
Acc.$\uparrow$ & \#Tok.$\downarrow$ &
Acc.$\uparrow$ & \#Tok.$\downarrow$ &
Acc.$\uparrow$ & \#Tok.$\downarrow$ &
Acc.$\uparrow$ & \#Tok.$\downarrow$ \\
\midrule

\multirow{7}{*}[-1.0ex]{%
  \rotatebox[origin=c]{90}{%
    \shortstack{Qwen2.5-3B\\Instruct}}}
& Full Text
& 46.0 & 24,518
& 21.9 & 124,317
& 12.9 & 1,893
& \cellcolor{tablesecond}26.6 & 145,521
& 18.8 & 58,803 \\


& xRAG
& 36.0 & 28,714
& 55.8 & 121,118
& 8.5 & 1,693
& 10.2 & 115,744
& 15.9 & 50,040 \\

& AutoCompressor$^{\dagger}$
& 32.0 & 25,006
& 30.5 & 123,999
& -- & --
& -- & --
& -- & -- \\

& MemGen
& 54.0 & 49,510
& \cellcolor{tablesecond}66.1 & 123,302
& \cellcolor{tablesecond}13.3 & 19,386
& 3.8 & 208,378
& 18.7 & 86,809 \\

& GRPO
& \cellcolor{tablesecond}68.0 & 27,297
& 58.4 &  123,410
& 11.3 & 1,893
& 26.0 & 145,521
& \cellcolor{tablesecond}23.2 & 58,762 \\

& OPSD
& 54.0 & 27,297
& 29.6 & 123,410
& 12.8 & 1,893
& 26.2 & 145,521
& 19.9 & 58,762 \\

& \mflogo\,\textbf{\texttt{MemFold}}
& \cellcolor{tablebest}\textbf{70.0} & 24,395
& \cellcolor{tablebest}\textbf{88.4} & 123,229
& \cellcolor{tablebest}\textbf{19.9} & 2,692
& \cellcolor{tablebest}\textbf{32.4} & 122,545
& \cellcolor{tablebest}\textbf{33.8} & 52,662 \\

\midrule

\multirow{7}{*}[-1.0ex]{%
  \rotatebox[origin=c]{90}{%
    \shortstack{Qwen2.5-7B\\Instruct}}}
& Full Text
& 60.0 & 24,517
& 24.0 & 124,317
& \cellcolor{tablesecond}14.0 & 1,893
& \cellcolor{tablesecond}25.4 & 145,532
& 19.8 & 58,806 \\


& xRAG
& 62.0 & 28,733
& 64.9 & 121,118
& 10.9 & 1,703
& 12.2 & 115,754
& 19.8 & 50,049 \\

& AutoCompressor$^{\dagger}$
& 66.0 & 25,506
& 30.7 & 123,999
& -- & --
& 9.6 & 109,030
& -- & -- \\

& MemGen
& \cellcolor{tablesecond}76.0 & 49,510
& \cellcolor{tablesecond}78.5 & 123,302
& \cellcolor{tablebest}14.1 & 20,002
& 11.0 & 178,477
& 23.4 & 78,769 \\

& GRPO
& 70.0 & 27,297
& 62.2 & 123,411
& \cellcolor{tablebest}14.1 & 1,893
& 25.0 & 145,535
& \cellcolor{tablesecond}25.0 & 58,766 \\

& OPSD
& 64.0 & 27,297
& 47.2 & 123,420
& 13.2 & 1,895
& 25.0 & 145,542
& 22.4 & 58,770 \\

& \mflogo\,\textbf{\texttt{MemFold}}
& \cellcolor{tablebest}\textbf{88.0} & 24,851
& \cellcolor{tablebest}\textbf{94.4} & 124,640
& \cellcolor{tablebest}14.1 & 2,192
& \cellcolor{tablebest}\textbf{36.8} & 122,089
& \cellcolor{tablebest}\textbf{33.0} & 52,451 \\

\midrule

\multirow{7}{*}[-1.0ex]{%
  \rotatebox[origin=c]{90}{Qwen3-4B}}
& Full Text
& 56.0 & 24,500
& 4.3 & 124,301
& 13.6 & 1,894
& \cellcolor{tablesecond}27.8 & 145,506
& 17.6 & 58,796 \\


& xRAG
& 66.0 & 28,934
& 65.2 & 121,146
& 2.8 & 1,827
& 6.6 & 115,925
& 13.8 & 50,176 \\

& AutoCompressor$^{\dagger}$
& 58.0 & 25,519
& 30.8 & 123,982
& -- & --
& 14.2 & 107,087
& -- & -- \\

& MemGen
& 56.0 & 49,510
& \cellcolor{tablesecond}65.7 & 124,302
& 12.3 & 19,717
& 10.4 & 112,881
& 20.0 & 60,345 \\

& GRPO
& 62.0 & 24,500
& \cellcolor{tablesecond}65.7 & 123,393
& \cellcolor{tablesecond}13.8 & 1,895
& 27.4 & 145,527
& \cellcolor{tablesecond}25.7 & 58,684 \\

& OPSD
& \cellcolor{tablesecond}74.0 & 24,500
& 39.5 & 123,441
& \cellcolor{tablesecond}13.8 & 1,895
& 26.0 & 145,527
& 22.3 & 58,691 \\

& \mflogo\,\textbf{\texttt{MemFold}}
& \cellcolor{tablebest}\textbf{84.0} & 24,395
& \cellcolor{tablebest}\textbf{89.4} & 124,155
& \cellcolor{tablebest}\textbf{15.2} & 2,268
& \cellcolor{tablebest}\textbf{38.6} & 123,488
& \cellcolor{tablebest}\textbf{33.4} & 52,810 \\

\bottomrule
\end{NiceTabular}%
}

\par

\begin{minipage}{\textwidth}
\footnotesize
\raggedright
\setlength{\parindent}{0pt}
\hangindent=1.2em
\hangafter=1
\textsuperscript{\(\dagger\)}\,
AutoCompressor yields no valid outputs in these settings: its
PersonaMem-32K-trained checkpoints return a single prediction on PrefEval, while the 3B checkpoint produces malformed LongMemEval outputs.
\end{minipage}

\end{table*}

\subsection{Results}

Table~\ref{tab:main_tab} reports accuracy and average end-to-end tokens per test instance for every method under a shared backbone and evaluation protocol. \textbf{\texttt{MemFold}} obtains the highest accuracy on every benchmark and backbone, tying the strongest baseline in one cell on PrefEval. Two patterns in the table matter more than the individual entries. First, the advantage widens with history length: the margin over the best competing method is larger at PersonaMem-128K than at 32K for all three backbones, and several baselines that are competitive at 32K fall sharply at 128K, whereas \textbf{\texttt{MemFold}} does not, even though its reader is given the same $K$ vectors in both settings and only the history behind them grows. Second, the accuracy does not come from spending more at inference: average end-to-end cost stays close to full-context inference and is lower in aggregate for all three backbones. The consistent exception is PrefEval, where the history is too short for reader-side savings to offset memory-construction costs. xRAG uses slightly fewer tokens, but at a substantial accuracy cost. What the table cannot show is where the gain originates, or whether the reader is using the memory at all; we take those up in Sections~\ref{sec:ablation} and~\ref{sec:discussion}, respectively.

\section{Ablation Study}
\label{sec:ablation}

Table~\ref{tab:component_ablation} isolates the contributions of
memory-interface initialization and the two on-policy objectives. Reader initialization produces the largest initialization gain, indicating that the model must first learn to consume the soft-memory interface before on-policy optimization can be effective. Writer initialization provides an additional gain by improving the textual memory from which the soft representation is constructed. Among the on-policy objectives, GRPO accounts for most of the task improvement, while OPD provides a complementary gain through token-level guidance from the textual-memory teacher when combined with GRPO. Finally, the text-space control performs comparably but does not outperform the full soft-memory model, suggesting that fixed-budget compression is not the primary bottleneck under this training configuration.

\begin{table*}[!ht]
\centering
\caption{\small
Ablation of initialization and on-policy objectives. The text-space control uses textual memory, whereas \textbf{\texttt{MemFold}} uses fixed-budget soft memory.
\texorpdfstring{\cmark}{full} = satisfies,
\texorpdfstring{\xmark}{absent} = does not satisfy,
and \textit{N/A} = not applicable.
}

\label{tab:component_ablation}
\small
\setlength{\tabcolsep}{4.5pt}
\renewcommand{\arraystretch}{1.08}
\resizebox{\textwidth}{!}{%
\begin{tabular}{lccccccccc}
\toprule
\multirow{2}{*}{\textbf{Method}} &
\multicolumn{2}{c}{\textbf{Initialization}} &
\multicolumn{2}{c}{\textbf{On-Policy Objectives}} &
\multirow{2}{*}{\shortstack{\textbf{Memory}\\\textbf{Interface}}} &
\multicolumn{2}{c}{\textbf{PersonaMem-32K}} &
\multicolumn{2}{c}{\textbf{PersonaMem-128K}} \\
\cmidrule(lr){2-3}
\cmidrule(lr){4-5}
\cmidrule(lr){7-8}
\cmidrule(lr){9-10}
&
\textbf{Writer} &
\textbf{Reader} &
\textbf{OPD} &
\textbf{GRPO} &
&
\textbf{Mean $\uparrow$} &
\textbf{Pass@16 $\uparrow$} &
\textbf{Mean $\uparrow$} &
\textbf{Pass@16 $\uparrow$} \\
\midrule
w/o Writer Initialization
& \xmark & \cmark & \cmark & \cmark
& Soft
& 70.8 & 82.0 & 80.0 & 85.7 \\

w/o Reader Initialization
& \cmark & \xmark & \cmark & \cmark
& Soft
& 47.9 & 76.0 & 46.2 & 67.0 \\

w/o OPD
& \cmark & \cmark & \xmark & \cmark
& Soft
& 71.9 & 84.0 & 86.9 & 91.8 \\

w/o GRPO
& \cmark & \cmark & \cmark & \xmark
& Soft
& 58.6 & 78.0 & 70.2 & 82.3 \\

Initialization Only
& \cmark & \cmark & \xmark & \xmark
& Soft
& 60.5 & 82.0 & 70.8 & 85.1 \\

\midrule
Text-Space Control
& \cmark & \textit{N/A} & \cmark & \cmark
& Text
& 69.3 & \textbf{88.0} & 85.5 & 91.9 \\

\mflogo\,\textbf{\texttt{MemFold}}
& \cmark & \cmark & \cmark & \cmark
& Soft
& \textbf{75.4} & \textbf{88.0}
& \textbf{87.9} & \textbf{93.6} \\
\bottomrule
\end{tabular}%
}
\end{table*}

\section{Discussion}
\label{sec:discussion}

We examine how the learned interface transfers across datasets, how efficiently
it trains, how its budget affects accuracy and cost, and whether the reader uses
instance-specific memory.
A final example illustrates a failure of personalization in a single response. 

\begin{questionbox}
Does a soft-memory interface learned on one dataset still work when the
dataset, the evaluation format, and the kind of memory the task demands all
change?
\end{questionbox}

We consider two OOD settings that target different memory capabilities. For PrefEval, we directly evaluate frozen PersonaMem-32K checkpoints on its implicit-persona subset. This setting preserves the underlying task of modeling user preferences while changing the dataset distribution and evaluation format, thereby measuring cross-dataset personalization transfer. For LongMemEval, we evaluate checkpoints trained and selected only on LoCoMo. This setting tests a broader form of long-context memory generalization involving multi-session reasoning, temporal relations, and knowledge updates. As shown in Table~\ref{tab:main_tab}, source-domain improvements do not necessarily translate into stronger OOD performance. On PrefEval, GRPO and OPSD provide inconsistent gains and can even underperform the untrained Full Text baseline for some backbones, suggesting that their optimization may specialize to the source-domain supervision and input format. The latent-memory baselines perform poorly on LongMemEval, though not uniformly on PrefEval, where MemGen remains competitive. In contrast, \textbf{\texttt{MemFold}} maintains more robust performance across the personalization and long-context OOD settings, indicating that the learned soft-memory interface remains transferable under dataset-distribution shifts. The transfer is not uniform: on PrefEval with Qwen2.5-7B-Instruct, \textbf{\texttt{MemFold}} only ties the strongest baselines, and the methods occupy a narrow accuracy range, making this setting less discriminative than the others.

\begin{questionbox}
Does the learned memory interface also improve training efficiency?
\end{questionbox}

Figure~\ref{fig:exp} compares training with Qwen2.5-3B-Instruct on PersonaMem-32K.
\textbf{\texttt{MemFold}} improves accuracy more rapidly with optimizer updates
and reaches comparable accuracy with fewer student rollouts. Its reward and
distillation objectives reuse the same responses, while the teacher scores
sampled tokens without autoregressive generation. These curves compare complete
method configurations, including different inputs and initializations
(Appendix~\ref{app:efficiency_comparison}). The observed advantage therefore
reflects the full training recipe. The curves measure optimization and rollout
efficiency; total training cost also includes initialization and teacher forward
passes.

\Needspace{17\baselineskip}
\begin{questionbox}
How much does the outcome depend on the particular budget $K$ we chose, in
accuracy and in what it costs to run?
\end{questionbox}

\begin{wrapfigure}[13]{r}{0.30\textwidth}

\centering
\captionsetup{font=small,justification=raggedright,singlelinecheck=false,skip=3pt}
\includegraphics[width=\linewidth]{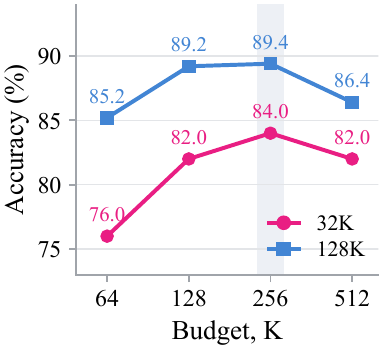}
\caption{Accuracy across budgets and history lengths (Qwen3-4B).}
\label{fig:budget_accuracy}
\end{wrapfigure}
\textbf{Sensitivity of accuracy to the budget.}
Increasing $K$ eightfold from 64 to 512 does not produce monotonic gains
(Figure~\ref{fig:budget_accuracy}). Both curves are single-peaked at $K=256$
and nearly flat from 128 upward, so only the smallest budget is clearly
under-provisioned and the useful range is broad rather than a sharp optimum.
The longer histories give a flatter curve, as expected if the writer is already
discarding most of the history before compression, so that adding vectors
changes what survives less than it changes how much. The decline at 512 is the
more informative end. Extra capacity is not free here, and because the
compressor and reader are initialized under a fixed budget that we did not
retune per $K$, we read that decline as a property of this training recipe
rather than as evidence that more vectors carry less.

\Needspace{13\baselineskip}
\begin{wrapfigure}[13]{r}{0.30\textwidth}

\centering
\captionsetup{font=small,justification=raggedright,singlelinecheck=false,skip=0pt}
\includegraphics[width=0.97\linewidth]{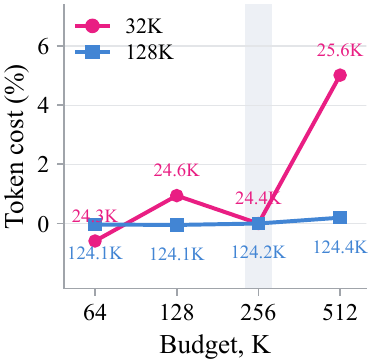}
\caption{Token cost relative to $K=256$ (Qwen3-4B); labels show absolute counts.}
\label{fig:budget_cost}
\end{wrapfigure}
\textbf{Cost of the budget.}
End-to-end token cost barely moves across budgets
(Figure~\ref{fig:budget_cost}), varying by less than one percent on the longer
benchmark. The change is not monotonic in $K$: increasing the budget from 64
to 512 adds 448 vectors to the reader's input, roughly a third of the token
difference measured on PersonaMem-32K. Most of the variation therefore comes
from generated tokens. Every configuration still reads the interaction history
once to build the textual memory; on the longer benchmark, that pass accounts
for all but a fraction of a percent of the per-instance cost. Enlarging the
budget therefore costs little, and shrinking it saves little. A compact memory
keeps the reader's interface independent of history length, but does not remove
the cost of reading the full interaction history for each instance.

\Needspace{18\baselineskip}
\begin{questionbox}
Is the reader actually using the memory built for this instance, or answering
from what the task alone makes likely?
\end{questionbox}

\begin{wrapfigure}[14]{r}{0.30\textwidth}

\centering
\captionsetup{font=small,justification=raggedright,singlelinecheck=false,skip=8pt}
\includegraphics[width=\linewidth]{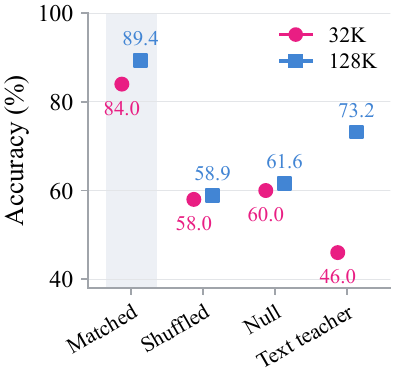}
\caption{Accuracy across memory conditions and the text teacher (Qwen3-4B).}
\label{fig:memory_reliance}
\end{wrapfigure}
Replacing the matched soft memory with shuffled or null memory causes substantial accuracy drops on both datasets (Figure~\ref{fig:memory_reliance}). The model therefore depends on the content of the memory built for this instance, rather than ignoring it or answering from task priors alone. The intervention establishes that the memory is used; it does not establish which parts of it are used, and in particular it does not show that the reader resolves the correct time-dependent version of a preference when the history contains several. The textual-memory teacher is not an accuracy oracle: its task accuracy is well below the final student's. This is consistent with how $\mathcal{L}_{\mathrm{OPD}}$ uses it (Section~\ref{sec:on_policy_optimization}). The gate weights only tokens the student has already sampled, and its influence on any single token is bounded, so the teacher can reinforce tokens the student under-reads from its memory without its lower accuracy dominating the update. GRPO drives most of the task-level gain (Section~\ref{sec:ablation}).

\begin{questionbox}
What does a personalization failure actually look like in a response, when the
aggregate scores only say that one was produced?
\end{questionbox}

Aggregate scores do not reveal localized personalization failures. Figure~\ref{fig:prefeval_case} shows one such case: although the GRPO and OPSD responses are generally relevant, both recommend wearable trackers that contradict the user's preference. In contrast, \textbf{\texttt{MemFold}} provides preference-consistent alternatives, including journaling, body measurements, and progress photos. This example illustrates a failure pattern rather than its frequency.

\begin{figure*}[h]
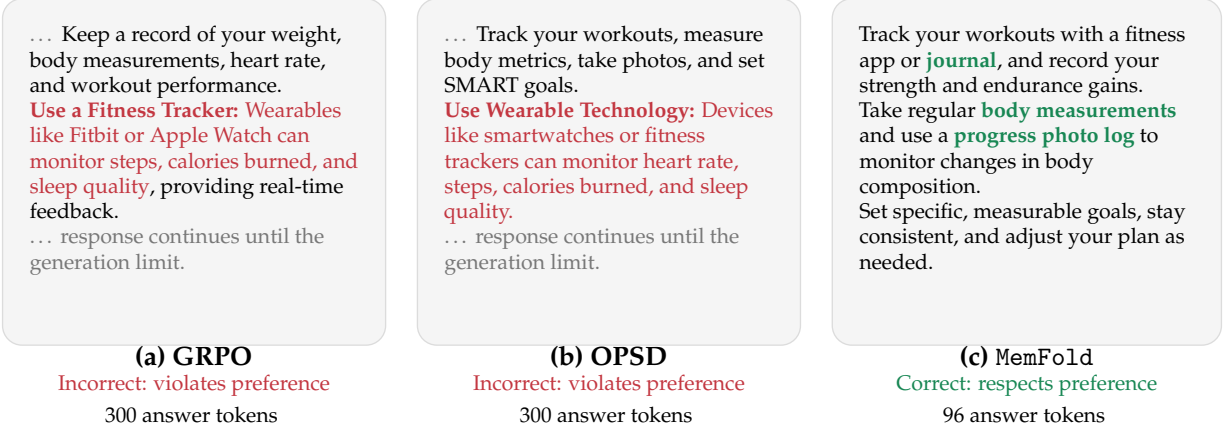

\centering

\small
\begin{minipage}{0.98\textwidth}
\textbf{User preference:}
\textit{``I dislike using wearable technology and fitness trackers.''}

\textbf{Query:}
\textit{``Can you recommend some effective ways for me to monitor my
fitness progress?''}
\end{minipage}

\begin{minipage}[t]{0.315\textwidth}
\begin{responsebox}
\scriptsize

\textcolor{casemuted}{\ldots}
Keep a record of your weight, body measurements, heart rate, and workout
performance.

\textcolor{casebad}{
\textbf{Use a Fitness Tracker:}
Wearables like Fitbit or Apple Watch can monitor steps, calories burned,
and sleep quality}, providing real-time feedback.

\textcolor{casemuted}{
\ldots\ response continues until the generation limit.
}
\end{responsebox}

\centering
\textbf{(a) GRPO}

{\scriptsize
\textcolor{casebad}{Incorrect: violates preference}\\
300 answer tokens
}
\end{minipage}
\hfill
\begin{minipage}[t]{0.315\textwidth}
\begin{responsebox}
\scriptsize

\textcolor{casemuted}{\ldots}
Track your workouts, measure body metrics, take photos, and set SMART goals.

\textcolor{casebad}{
\textbf{Use Wearable Technology:}
Devices like smartwatches or fitness trackers can monitor heart rate,
steps, calories burned, and sleep quality.
}

\textcolor{casemuted}{
\ldots\ response continues until the generation limit.
}
\end{responsebox}

\centering
\textbf{(b) OPSD}

{\scriptsize
\textcolor{casebad}{Incorrect: violates preference}\\
300 answer tokens
}
\end{minipage}
\hfill
\begin{minipage}[t]{0.320\textwidth}
\begin{responsebox}
\scriptsize

Track your workouts with a fitness app or
\textcolor{casegood}{\textbf{journal}}, and record your strength and
endurance gains.

Take regular
\textcolor{casegood}{\textbf{body measurements}}
and use a
\textcolor{casegood}{\textbf{progress photo log}}
to monitor changes in body composition.

Set specific, measurable goals, stay consistent, and adjust your plan
as needed.
\end{responsebox}

\centering
\textbf{(c) \textbf{\texttt{MemFold}}}

{\scriptsize
\textcolor{casegood}{Correct: respects preference}\\
96 answer tokens
}
\end{minipage}

\caption{
A PrefEval implicit-persona example with Qwen3-4B. The displayed preference
summarizes the interaction history and is not given explicitly to the models.
GRPO and OPSD recommend wearable trackers, whereas
\textbf{\texttt{MemFold}} respects the preference by suggesting non-wearable
alternatives. Colored spans mark preference-relevant content. Token counts
cover generated answers only; this example illustrates a failure mode, not its
frequency or end-to-end efficiency.
}
\label{fig:prefeval_case}

\end{figure*}

\section{Conclusion}

We presented \textbf{\texttt{MemFold}}, which judges a compact personalized memory
by the generations it supports rather than by the text it reconstructs. A
query-conditioned textual memory is compressed into $K$ soft vectors forming a
fixed-budget reader interface, and the reader is optimized on its own rollouts
under two complementary signals: group-relative rewards for task outcomes, and
confidence-gated on-policy distillation in which a frozen textual-memory teacher
re-scores the student's sampled tokens. Across three backbones it achieves the
highest accuracy we measured on PersonaMem-32K and PersonaMem-128K, transfers to
PrefEval and LongMemEval without target-domain training, and still depends on
instance-specific memory under shuffled- and null-memory interventions.

\clearpage
\bibliographystyle{unsrtnat}
\bibliography{iclr2027_conference}

\newpage
\appendix
\section{Implementation Details}

\subsection{\textbf{\texttt{MemFold}}}

\subsubsection{Textual Memory Construction}
\label{app:memory_extraction}

\textbf{Memory format.} For each context--query pair, the external
extractor produces a structured memory $M$ consisting of three fields:
grounded evidence, temporal relations, and derived facts. Memories are
extracted by GLM-5.2~\citep{zeng2026glm} from the history and the question
only; the reference answer is never provided to the extractor. We retain
only annotations that are grounded in the interaction history and relevant
to the corresponding query. The adapter is trained on the resulting
context--query--memory examples and subsequently writes $M$ without the
external extractor; we denote the adapter after this stage by $\theta_w$.

\paragraph{Example.}
\label{app:memory_example}
Figure~\ref{fig:memory_example} presents a representative memory template,
illustrating how evidence, temporal relations, and derived facts are organized
within the structured textual memory.

\begin{figure}[h]
\centering

\begin{memorybox}{A Representative Extracted Memory}
\footnotesize

\textbf{Evidence}
\begin{itemize}
    \setlength{\itemsep}{1pt}
    \setlength{\parskip}{0pt}
    \setlength{\parsep}{0pt}

    \item The user organized a volunteering day at a local shelter
    for the whole family to help the shelter and show the kids the
    importance of supporting those in need.

    \item The user later decided against future family volunteering
    after observing that the children felt resentful, and instead
    chose to prioritize activities aligned with their own interests.

    \item The user started a family photo challenge with weekly themes
    to encourage creativity and discussion among family members.

    \item The user enjoys writing a personal finance blog and views
    writing as a more suitable way to influence and support others.
\end{itemize}

\textbf{Temporal Relations}
\begin{itemize}
    \setlength{\itemsep}{1pt}
    \setlength{\parskip}{0pt}
    \setlength{\parsep}{0pt}

    \item The user first encouraged structured family volunteering,
    but later shifted toward child-led activities after observing
    reduced engagement.

    \item The user initially participated in a financial literacy
    volunteer program, but later concluded that blogging better
    matched their preferred way of helping others.
\end{itemize}

\textbf{Derived Facts}
\begin{itemize}
    \setlength{\itemsep}{1pt}
    \setlength{\parskip}{0pt}
    \setlength{\parsep}{0pt}

    \item The user currently prioritizes child-led and interest-based
    activities for maintaining positive family dynamics.

    \item The user prefers teaching and influencing others through
    writing and creative projects rather than formal programs.
\end{itemize}

\end{memorybox}

\caption{
A representative structured textual memory extracted from a long
PersonaMem interaction history. For readability, we show a representative
subset of the extracted entries.
}
\label{fig:memory_example}
\end{figure}

\subsubsection{Soft-Memory Construction and Initialization}
\label{app:stage2}

\paragraph{Compressor architecture.}
We serialize the fields of textual memory $M$ in a fixed order and encode the
resulting sequence with a frozen backbone encoder. The sequence is processed in
chunks of at most 2,048 tokens, with every 32 consecutive hidden states
mean-pooled to form the compressor input $H$. A two-layer Perceiver-style
compressor with $K$ learned queries, a latent dimension of 768, and 12 attention
heads aggregates these representations. A backbone-specific projector maps the
compressed states into the reader's input embedding space:
\[
\comp(H)
=
P_{\phi}
\left(
    \operatorname{Comp}_{\phi}(Q_K,H)
\right)
\in\mathbb{R}^{K\times d},
\]
where $\comp$ denotes the composition of the compressor and projector.
Unless otherwise specified, we use $K=256$ and train a separate compressor for
each dataset--backbone configuration.

\paragraph{Hidden-state layer selection.}
We select the encoder layer using a frozen, probe-free retrieval experiment
over 3,347 memory records from 272 sessions. Candidate layers are evaluated on
their ability to distinguish competing attribute values and temporal updates.
Selection uses the mean development AUC of the two tasks, while the held-out
split is used only for reporting. As shown in
Table~\ref{tab:layer_selection}, the fourth Transformer block achieves the
highest development score for all three backbones and consistently performs
best on held-out update discrimination. We therefore use its output as the
compressor input.

\begin{table}[tbp]
\centering
\caption{
Hidden-state layer selection. Dev. Score is the mean development AUC over
value and update discrimination; the remaining metrics are measured on the
held-out update task. Bold indicates the selected layer.
}
\label{tab:layer_selection}
\small
\setlength{\tabcolsep}{5pt}
\renewcommand{\arraystretch}{1.08}

\begin{tabular}{@{}llrrrr@{}}
\toprule
\textbf{Backbone}
& \textbf{Layer}
& \textbf{Dev. Score}
& \textbf{Update Top-1}
& \textbf{Update AUC}
& \textbf{Margin} \\
\midrule

\multirow{5}{*}{Qwen2.5-3B}
& L1
& 98.98
& 96.12
& 97.55
& 0.207 \\

& \textbf{L4}
& \textbf{100.00}
& \textbf{100.00}
& \textbf{100.00}
& \textbf{0.319} \\

& L12
& 94.39
& 79.61
& 90.37
& 0.091 \\

& L24
& 86.13
& 67.96
& 79.04
& 0.042 \\

& L36 (final)
& 77.49
& 37.86
& 51.76
& $-0.041$ \\

\midrule

\multirow{5}{*}{Qwen3-4B}
& L1
& 95.92
& 93.20
& 96.35
& 0.148 \\

& \textbf{L4}
& \textbf{98.98}
& \textbf{98.06}
& \textbf{99.19}
& \textbf{0.298} \\

& L12
& 96.94
& 87.38
& 90.27
& 0.151 \\

& L24
& 84.85
& 60.19
& 72.15
& 0.038 \\

& L36 (final)
& 77.29
& 41.75
& 55.48
& $-0.021$ \\

\midrule

\multirow{5}{*}{Qwen2.5-7B}
& L1
& 97.62
& 96.12
& 97.78
& 0.266 \\

& \textbf{L4}
& \textbf{98.98}
& \textbf{100.00}
& \textbf{100.00}
& \textbf{0.321} \\

& L12
& 87.96
& 69.90
& 80.04
& 0.078 \\

& L24
& 76.95
& 33.98
& 48.29
& $-0.037$ \\

& L28 (final)
& 73.69
& 33.01
& 47.60
& $-0.024$ \\

\bottomrule
\end{tabular}
\end{table}

\paragraph{Initialization overview.}
The soft-memory interface is initialized through four sequential procedures:
compressor reconstruction, representation warmup, auxiliary reasoning
adaptation, and reader initialization. These procedures optimize distinct
objectives rather than a single combined initialization objective. Their optimization
settings and trainable modules are summarized in
Table~\ref{tab:train_config}.

For example $i$, let $c_i$ denote its shared-context identifier, $h_i$ its
permitted history prefix, $q_i$ its question, $o_i$ its answer options,
$m_i$ its textual memory, $s_i$ its auxiliary reasoning target, and $a_i$
its answer target. During compressor reconstruction and representation warmup, the soft
representation is computed from cached history-prefix states:
\[
Z_i=\comp(E(h_i)).
\]
During auxiliary reasoning adaptation and reader initialization, it is computed
from textual-memory states:
\[
Z_i=\comp(E(m_i)),
\]
where $E$ is the frozen backbone encoder. Cached encoder states are treated as
constants, and gradients propagate through $\comp$.

We obtain a context-level unit representation by averaging over the $K$ soft
positions and applying $\ell_2$ normalization:
\[
u_i
=
\operatorname{norm}_2
\left(
    \frac{1}{K}\sum_{k=1}^{K}Z_{ik}
\right).
\]
For examples $i$ and $j$ from different shared contexts, the separation loss is
\[
\mathcal L_{\mathrm{sep}}(i,j)
=
\left[
    \sqrt{2-2\kappa}
    -
    \lVert u_i-u_j\rVert_2
\right]_+,
\qquad
\kappa=0.8.
\]
All cross-entropy terms below are averaged over supervised target positions;
soft prefixes, prompts, and padding positions are excluded.

\paragraph{Compressor reconstruction.}
The frozen decoder reconstructs the textual memory from the compressed
history representation:
\[
\mathcal L_{\mathrm{rec}}(i)
=
-\frac{1}{|m_i|}
\sum_{t=1}^{|m_i|}
\log
p_{\theta_0}
\left(
    m_{it}
    \mid
    Z_i,q_i,m_{i,<t}
\right).
\]
The complete reconstruction objective is
\[
\boxed{
\mathcal L_{\mathrm{reconstruction}}
=
\mathbb E_i
\left[
    \mathcal L_{\mathrm{rec}}(i)
    +
    \mathcal L_{\mathrm{sep}}(i,j)
\right],
}
\]
where $j$ is sampled from a different shared context. The encoder and decoder
remain frozen, while the full compressor and projector are updated. No answer
or reasoning supervision is used in this procedure.

\paragraph{Representation warmup.}
Representation warmup operates on two cached history-state views per context
and uses no decoder or answer supervision. It combines separation across
different contexts, alignment between views of the same context, and
decorrelation among context prototypes:
\[
\boxed{
\mathcal L_{\mathrm{warmup}}
=
\mathcal L_{\mathrm{sep}}^{\mathrm{all}}
+
0.1\,\mathcal L_{\mathrm{align}}
+
0.1\,\mathcal L_{\mathrm{Gram}}.
}
\]
Here, $\mathcal L_{\mathrm{sep}}^{\mathrm{all}}$ averages the separation hinge
over different-context pairs, $\mathcal L_{\mathrm{align}}$ is the mean cosine
distance between same-context views, and $\mathcal L_{\mathrm{Gram}}$ penalizes
squared off-diagonal similarities between normalized context prototypes. The
entire compressor and projector are updated.

\paragraph{Auxiliary reasoning adaptation.}
The third procedure adapts the compressor using textual-memory states and
evidence-grounded reasoning targets. Let
\[
\ell_i(Z) = -\frac{1}{|s_i|}\sum_{t=1}^{|s_i|}
\log p_{\theta_0,\omega}\!\left(s_{i,t}\mid Z, q_i, o_i, s_{i,<t}\right)
\]
denote the reasoning-target negative log-likelihood. For a mismatched memory
from a different context, we define
\[
\mathcal L_{\mathrm{rank}}(i,j)
=
\left[
    \delta
    +
    \ell_i(Z_i)
    -
    \ell_i(Z_j)
\right]_+.
\]
The question, options, and reasoning target remain fixed, so only the soft
memory is replaced in the negative example. The complete objective is
\[
\boxed{
\mathcal L_{\mathrm{aux}}
=
\mathbb E_i
\left[
    \ell_i(Z_i)
    +
    \lambda_{\mathrm{rank}}
    \mathcal L_{\mathrm{rank}}(i,j)
    +
    0.1\,\mathcal L_{\mathrm{sep}}(i,j)
\right].
}
\]
We use
$(\lambda_{\mathrm{rank}},\delta)=(0.2,0.05)$ in the first epoch and
$(1.0,0.1)$ in the remaining epochs. This procedure updates the full compressor and projector together with a
temporary rank-8 LoRA $\omega$ on the frozen backbone. $\omega$ is
discarded afterward; only the adapted compressor and projector are
transferred to reader initialization.

\paragraph{Reader initialization.}
The final initialization procedure trains the reader to generate gold
answers from self-generated textual memories. Specifically, the
writer-initialized adapter $\theta_w$ generates and caches
\(
M_i^{w}
=
e_{\theta_w}(C_i,q_i).
\)

The adapter continues training from $\theta_w$. Let $\phi_0$ denote
the compressor parameters at the beginning of reader initialization,
and retain a frozen reference copy. The trainable and reference soft
representations are

\[
Z_i
=
C_{\phi}\!\left(E(M_i^{w})\right),
\qquad
Z_i^{0}
=
\operatorname{sg}\!\left[
C_{\phi_0}\!\left(E(M_i^{w})\right)
\right].
\]
The answer objective is
\begin{equation*}
\mathcal{L}_{\mathrm{answer}}(i) = -\frac{1}{|a_i|}\sum_{t=1}^{|a_i|}
\log \pi_\theta\!\left(a_{i,t}\mid Z_i, q_i, o_i, a_{i,<t}\right).
\end{equation*}
To limit drift in the soft-memory representation, we use the normalized
anchoring loss
\begin{equation*}
\mathcal{L}_{\mathrm{anchor}}(i) =
\frac{\frac{1}{Kd}\,\|Z_i - Z_i^0\|_F^2}
{\max\!\left(\frac{1}{Kd}\,\|Z_i^0\|_F^2,\ 10^{-8}\right)}.
\end{equation*}
The complete reader-initialization objective is
\begin{equation*}
\mathcal{L}_{\mathrm{reader\text{-}init}} =
\mathbb{E}_i\!\left[\mathcal{L}_{\mathrm{answer}}(i) + 0.1\,\mathcal{L}_{\mathrm{anchor}}(i)\right].
\end{equation*}
Ranking, auxiliary reasoning supervision, and writer replay are disabled
in this procedure. We update the reader LoRA $\theta$, the final resampler
layer, the resampler's output normalization layer, and the projector,
while keeping the backbone $\theta_0$, the remaining compressor
parameters, and the reference compressor frozen. The resulting adapter, $\theta_{\mathrm{init}}$, initializes
on-policy optimization and, when kept frozen, serves as the textual-memory
teacher $\pi_T$. Before on-policy optimization begins, it is also used
once to generate a new fixed training-memory snapshot,
\(
M_i^{\mathrm{init}}
=
e_{\theta_{\mathrm{init}}}(C_i,q_i),
\)

which is cached and held fixed throughout on-policy training.

\subsubsection{On-Policy Training Details}
\label{app:stage3}

\paragraph{Memory snapshots.}
For reader initialization, training memories are generated once by
the writer-initialized adapter $\theta_w$ using greedy decoding and
cached. After selecting the reader-initialization checkpoint
$\theta_{\mathrm{init}}$, we regenerate the training memories once
with $\theta_{\mathrm{init}}$:
\[
M^{\mathrm{init}}
=
e_{\theta_{\mathrm{init}}}(C,x).
\]
These memories and their encoder representations are cached and
remain fixed throughout on-policy optimization; they are not refreshed
after policy updates. The teacher reads the cached textual memory
$M^{\mathrm{init}}$, while the student receives
$Z=C_\phi(E(M^{\mathrm{init}}))$ compressed from the same memory.
At evaluation, memories are generated by the final adapter $\theta$,
so that a single adapter both writes and reads at test time.

\paragraph{OPD implementation.} In code, the OPD term is computed as $\bar g_{i,t}\,(\ell_{T,i,t}-\ell_{\theta,i,t})$, with both $\ell_{T,i,t}$ and $\bar g_{i,t}$ detached. Because $\ell_{T,i,t}$ carries no gradient, this differs from Eq.~\ref{eq:opd} only by a constant and yields the same gradient, $-\bar g_{i,t}\nabla_\theta \ell_{\theta,i,t}$; only the logged loss value differs.

\textbf{Rewards and advantages.} For PersonaMem, $r(y,a)=1$ if the parsed
option of $y$ matches the gold answer $a$, and $0$ otherwise. For LoCoMo,
we case-fold the generated and reference answers and extract word units
using \texttt{[\textbackslash w]+}, yielding sequences $S_y$ and $S_a$. Let
\begin{equation*}
O = \sum_{w} \min\!\big(\mathrm{count}_{S_y}(w),\, \mathrm{count}_{S_a}(w)\big).
\end{equation*}
The LoCoMo reward is
\begin{equation*}
r(y,a) = 0.75\,\frac{2O}{|S_y| + |S_a|} + 0.25\,\mathbb{1}\{S_y = S_a\},
\end{equation*}
and is set to zero if either sequence is empty. No additional format
penalty is applied. Group-relative advantages $\hat A_i$ follow
Section~\ref{sec:on_policy_optimization}; a group whose reward variance is at most
$10^{-6}$ is treated as having identical rewards, so all its $\hat A_i$
are set to $0$. Such a group contributes no GRPO gradient, while the
other enabled objectives remain active.

\paragraph{OPD token mask.}
OPD is computed only over student-generated response positions. The mask
excludes the prompt and soft-memory prefix, includes the first EOS token, and
excludes padding after EOS; if no EOS is generated, the full response is used.

\subsection{Training Configuration}
\label{app:training_configuration}

Table~\ref{tab:shared_hparams} summarizes the shared architectural and algorithmic hyperparameters used by \textbf{\texttt{MemFold}}. Unless otherwise specified, we use a fixed soft-memory budget of 256 tokens and extract textual-memory representations from the fourth Transformer block.

\begin{table}[h]
\centering
\caption{Shared architecture and optimization hyperparameters used by MemFold.}
\label{tab:shared_hparams}
\small
\begin{tabular}{@{}ll@{}}
\toprule
Hyperparameter & Setting \\
\midrule
\multicolumn{2}{@{}l}{\textit{Memory representation and compressor}} \\
Soft-memory budget $K$ & 256 soft tokens \\
Mean-pooling window & 32 hidden states \\
Compressor depth & 2 layers \\
Latent dimension & 768 \\
Attention heads & 12 \\
\midrule
\multicolumn{2}{@{}l}{\textit{Training and optimization}} \\
LoRA rank / scaling / dropout & 16 / 32 / 0.05 \\
OPD gate scale $\beta$ & 5.0 \\
GRPO clip range $\epsilon$ & 0.2 \\
\bottomrule
\end{tabular}
\end{table}

Table~\ref{tab:train_config} reports the procedure-specific optimization configuration for PersonaMem-32K. All procedures use AdamW with weight decay $0.01$ and global gradient clipping at $1.0$.

\begin{table}[h]
\centering
\caption{Training configuration for PersonaMem-32K.}
\label{tab:train_config}
\small
\setlength{\tabcolsep}{4pt}
\begin{tabular}{@{}lccp{3.6cm}@{}}
\toprule
Procedure & Learning Rate & Duration & Updated Modules \\
\midrule
\multicolumn{4}{@{}l}{\textit{Memory-interface initialization}} \\
Memory-writer initialization & $1\times10^{-5}$ & 3 epochs & Adapter \\
Compressor reconstruction & $3\times10^{-4}$ & 25 updates & Compressor, projector \\
Representation warmup & $1\times10^{-4}$ & 100 updates & Compressor, projector \\
Auxiliary reasoning adaptation & \begin{tabular}[t]{@{}c@{}}$1\times10^{-4}\to$\\$5\times10^{-5}/3\times10^{-5}$\end{tabular} & 3 epochs & \raggedright Temporary LoRA $\omega$, compressor, projector \tabularnewline
Reader initialization & $1\times10^{-5}$ & 3 epochs & \raggedright Adapter, selected compressor modules, projector \tabularnewline
\midrule
\multicolumn{4}{@{}l}{\textit{On-policy optimization}} \\
Joint OPD--GRPO optimization & $3\times10^{-7}$ & 3 epochs & Adapter \\
\bottomrule
\end{tabular}
\end{table}

Memory-writer initialization uses cosine learning-rate decay with a 5\% warmup, whereas all subsequent procedures use constant learning rates without warmup. Representation warmup combines cross-context separation, alignment, and Gram-matrix regularization, each with a weight of $0.1$. During auxiliary reasoning adaptation, the ranking weight increases from $0.2$ in the first epoch to $1.0$ in the remaining epochs, while the separation weight remains $0.1$. The corresponding ranking margins are $0.05$ and $0.1$, respectively. Reader initialization additionally applies a soft-output anchoring loss with weight $0.1$.

Memory-writer initialization, reader initialization, and on-policy
optimization train a single LoRA adapter on the query, key, value,
output, gate, up, and down projections. Memory-writer initialization
updates only this adapter, yielding $\theta_w$. Reader initialization
continues training the same adapter together with the final resampler
layer, resampler output normalization, and projector, while the backbone and remaining compressor parameters are frozen. The temporary rank-8 LoRA $\omega$ used for auxiliary reasoning adaptation is discarded before reader initialization. During on-policy optimization, only the adapter is updated; the backbone, soft-memory compressor, projector, and textual-memory teacher remain frozen.

For PersonaMem, on-policy rollouts use eight responses per query, a temperature of $1.0$, top-$p$ of $0.98$, and a maximum response length of five tokens. We set $\lambda_{\mathrm{OPD}}=0.02$ and $\lambda_{\mathrm{GRPO}}=0.3$. For LoCoMo, we use four responses per query, a temperature of $0.8$, top-$p$ of $0.95$, and a maximum response length of 64 tokens. Both OPD and GRPO are assigned a coefficient of $1.0$. We use no additional reference-policy or KL regularization.

Frozen-backbone computation uses BF16, whereas LoRA parameters and numerically sensitive log-probability and loss computations use FP32. Cached encoder representations are stored in FP16 and converted to BF16 when loaded. Memory-writer initialization uses FlashAttention-2, while reader initialization and on-policy optimization use SDPA. These procedures are trained with distributed data parallelism on four GPUs, whereas compressor reconstruction, representation warmup, and auxiliary reasoning adaptation use a single GPU. All experiments were conducted on NVIDIA H200 GPUs. In the PersonaMem-32K/Qwen2.5-3B efficiency experiment, on-policy optimization completed 369 updates on four GPUs in approximately 12.9 minutes of training wall-clock time, excluding evaluation and the preceding initialization procedures.

\subsection{Datasets \& Benchmarks}
\label{app:evaluation_protocols}

\noindent\textbf{PersonaMem-32K and PersonaMem-128K.} PersonaMem \citep{jiang2025know} evaluates whether models can track evolving user preferences and answer personalized multiple-choice questions based on long conversational histories. The 32K and 128K variants differ in context length, and our evaluation uses 50 and 233 test questions, respectively.

\noindent\textbf{PrefEval.} PrefEval \citep{zhao2025llmsrecognizepreferencesevaluating} evaluates whether models can infer and apply implicit user preferences from interaction histories; we use its 1,000-example implicit-persona subset to assess cross-dataset personalization generalization. We evaluate the PersonaMem-32K-trained checkpoints and score their predictions using the Gemini API.

\noindent\textbf{LongMemEval.} LongMemEval \citep{wu2024longmemeval} evaluates long-term interactive memory across tasks such as knowledge updates, multi-session reasoning, temporal reasoning, and user-preference recall. We use the cleaned 500-question LongMemEval-S split as a held-out benchmark and evaluate answers with \texttt{gemini-3.8-flash} using LOW thinking and structured yes/no outputs.






\subsection{Baseline Implementation Details}
\label{app:baseline_details}

All baselines use the same Qwen backbone, source-domain data splits,
visible-history boundary, and evaluation protocol within each
experimental setting. Hyperparameters and checkpoints are selected
only on the source-domain validation set, and all parameters remain
frozen during OOD evaluation. Because the original implementations of
xRAG, MemGen, and AutoCompressor do not directly support our Qwen
backbones and datasets, we describe them as adapted implementations
rather than exact reproductions of official checkpoints.

\noindent\textbf{xRAG.}
Our xRAG implementation preserves the original frozen-retriever,
frozen-language-model, and projector-only training design
\citep{cheng2024xrag}. A frozen GTE-large encoder retrieves individual
messages from the visible history, and a two-layer MLP with a GELU
activation maps each retrieved vector to one soft token in the Qwen
embedding space. We first train the projector to reconstruct
source-domain context messages and then fine-tune it on source-domain
QA while keeping both the retriever and Qwen backbone frozen. The
retrieval query, top-$k$, and learning rate are selected exclusively
using source-domain validation data.

\noindent\textbf{MemGen.}
Our MemGen implementation uses the official Weaver core
\citep{zhang2026memgen} with an interface adapted to Qwen and our data
formats. The Weaver generates latent memories that are injected into
the reasoner, using eight prompt latents, eight inference latents, and at most five inference-time augmentations. The Weaver, projection parameters, latent parameters, and their associated rank-16 LoRA modules are trained on the source-domain split. Histories exceeding the Weaver's input limit are split into chunks of 28,672 tokens so that every visible message is processed; no history is truncated. The reported PersonaMem and PrefEval checkpoints use the trained Weaver latent memory with the Trigger disabled, and therefore do not reflect
MemGen's full Trigger-based closed-loop procedure.

\noindent\textbf{AutoCompressor.}
We adapt the recurrent summary-token mechanism of AutoCompressor
\citep{chevalier2023adapting} to the Qwen backbones. Each visible
history is divided into 1,536-token segments, with every segment
producing 32 learned summary vectors that are carried forward to
compress subsequent segments. The final reader receives only the
accumulated summaries, system instruction, and question rather than
the compressed raw segments. We jointly train the summary embeddings
and rank-16 Qwen LoRA modules for one source-domain epoch, using a fixed
seed of 42 and without selecting the training duration based on test
performance.

\noindent\textbf{OPSD.}
We implement OPSD \citep{zhao2026self} using a fixed version of its
official trainer as a full-history QA baseline. The student receives the complete history visible before the query time $\tau$ together with the question, while the frozen teacher may additionally access the reference answer for the same source-domain training example, as specified by the method. Training
uses generalized Jensen--Shannon divergence with point-wise clipping
of $0.05$ and no task reward. The reader is trained with rank-64 LoRA,
a learning rate of $5\times10^{-6}$, and temperature $1.1$ for 500 steps; reference answers are never available during validation, OOD
evaluation, or test-time inference.

\noindent\textbf{Vanilla GRPO.}
Vanilla GRPO \citep{shao2024deepseekmath} is implemented with TRL as a
full-history QA baseline without teacher-distribution supervision.
For each source-domain question, the policy samples eight completions
and computes advantages using group-normalized task rewards without a
separate value model. PersonaMem uses strict multiple-choice exact
match as the reward, whereas open-ended source-domain QA uses
normalized token F1. We set the KL coefficient to zero, use rank-64
LoRA with a learning rate of $5\times10^{-6}$, and train for 1000 steps
at temperature $1.2$; evaluation uses the same deterministic decoding
protocol as the other methods.



\section{Evaluation Details}

\subsection{Efficiency Comparison Setup}
\label{app:efficiency_comparison}

Figure~\ref{fig:exp} compares the training dynamics of five complete method configurations using Qwen2.5-3B-Instruct on PersonaMem-32K. The figure is intended as a system-level comparison of optimization and rollout efficiency, rather than a controlled ablation in which only the loss function changes. SDPO~\citep{hubotter2026reinforcement} and GRPO+OPD appear only in this comparison; the latter applies our OPD objective to a full-history student without soft memory.

\begin{table}[h]
\centering
\caption{Configurations compared in Figure~\ref{fig:exp}.}
\label{tab:efficiency_configurations}
\small
\setlength{\tabcolsep}{4pt}
\renewcommand{\arraystretch}{1.08}

\begin{tabularx}{\linewidth}{
    @{}
    l
    >{\raggedright\arraybackslash}p{0.16\linewidth}
    >{\raggedright\arraybackslash}p{0.16\linewidth}
    >{\raggedright\arraybackslash}p{0.20\linewidth}
    >{\raggedright\arraybackslash}X
    @{}
}
\toprule
\textbf{Method}
& \textbf{Student Input}
& \textbf{Initialization}
& \textbf{Supervision}
& \textbf{Loss} \\
\midrule

GRPO
& Full interaction history
& Base model
& Task reward
& $\mathcal{L}_{\mathrm{GRPO}}$ \\
\addlinespace[3pt]

OPSD
& Full interaction history
& Base model
& On-policy self-distillation
& $\mathcal{L}_{\mathrm{OPSD}}$ \\
\addlinespace[3pt]

SDPO
& Full interaction history
& Base model
& Successful-rollout feedback; EMA teacher
& $\sum_t D_{\mathrm{KL}}(p_t\|q_t)$ \\
\addlinespace[3pt]

GRPO+OPD
& Full interaction history
& Base model
& Task reward; textual-memory teacher
& $\begin{aligned}[t]
  &\lambda_{\mathrm{GRPO}}\mathcal{L}_{\mathrm{GRPO}}\\[-2pt]
  &{}+\lambda_{\mathrm{OPD}}\mathcal{L}_{\mathrm{OPD}}
  \end{aligned}$ \\
\addlinespace[3pt]

\textbf{\texttt{MemFold}}
& $K=256$ soft-memory vectors
& Reader-initialization checkpoint
& Task reward; textual-memory teacher
& $\begin{aligned}[t]
  &\lambda_{\mathrm{GRPO}}\mathcal{L}_{\mathrm{GRPO}}\\[-2pt]
  &{}+\lambda_{\mathrm{OPD}}\mathcal{L}_{\mathrm{OPD}}
  \end{aligned}$ \\

\bottomrule
\end{tabularx}
\end{table}

The GRPO and OPD losses are defined in Eqs.~\ref{eq:grpo}
and~\ref{eq:opd}; for GRPO+OPD, the student conditions on the full
history in place of $Z$. $\mathcal{L}_{\mathrm{OPSD}}$ denotes the
generalized Jensen--Shannon objective with point-wise clipping described
in Appendix~\ref{app:baseline_details}.
For SDPO, the table shows the defining reverse-KL objective in
\citet[Eq.~(1)]{hubotter2026reinforcement}, with
$p_t=\pi_\theta(\cdot\mid C,x,y_{<t})$ and
$q_t=\operatorname{sg}[\pi_{\mathrm{EMA}}(\cdot\mid C,x,f,y_{<t})]$.
Here $f$ is successful-rollout feedback, and the teacher scores the
student's sampled prefixes with gradients stopped. The corresponding
distillation advantage for a candidate token $v$ is
$A_t^{\mathrm{SDPO}}(v)=\log q_t(v)-\log p_t(v)$.

\paragraph{Training-step and rollout accounting.}
A gradient update step denotes one optimizer update after gradient accumulation, rather than one microbatch. Cumulative student rollouts count the number of student responses generated up to each checkpoint. Repeated sampling of the same question contributes a new rollout, whereas teacher forward passes do not. Groups whose rewards are all identical still count toward cumulative rollouts, although they contribute no GRPO gradient. At the final checkpoint of 369 optimizer updates (16 rollouts per update), GRPO, GRPO+OPD, SDPO, and MemFold each produce 5,904 student rollouts, while OPSD produces 5,860 because its final batch is padded to full size and the padded entries are not counted as rollouts.
\paragraph{Evaluation and plotting.}
Each method is trained for a fixed 369 optimizer updates and evaluated at
steps 37, 74, 111, 148, 185, 222, 259, 296, 333, and 369. Training is not
stopped or checkpoint-selected using test performance. Every checkpoint is
evaluated on the same 50 PersonaMem-32K test questions using greedy decoding
under the same fixed set of 12 deterministic answer-option permutations.
The reported curves come from one training run per method; the permutations
are repeated evaluations of the same checkpoint rather than independent
training seeds. The final curves directly connect the measured checkpoints
without smoothing. Where confidence intervals are shown, they are obtained by
bootstrapping test questions and therefore do not represent variation across
training seeds. For the training rollouts in Figure~\ref{fig:exp}, all five
configurations use temperature $1.0$ and top-$p=1.0$.

\subsection{End-to-End Token-Equivalent Accounting}
\label{app:token_accounting}

Let $P$ and $A$ denote the textual reader prompt and generated answer,
respectively, measured using each method's actual prompt and tokenizer. We
count every discrete token processed or generated at each inference stage and
treat each soft, summary, or latent position as one token equivalent. Repeated
processing by different components is counted separately. We exclude padding,
training-only annotation and teacher computation, and evaluation-judge tokens;
cached computation and memory-construction costs are not amortized across
questions or trials. Reported values are averaged over questions within each
trial and then over five trials. This metric measures logical
token-equivalent processing, not FLOPs or wall-clock latency. Table~\ref{tab:token_accounting} summarizes the components counted for each method.

\begin{table}[htbp]
\centering
\caption{End-to-end token-equivalent accounting by method.}
\label{tab:token_accounting}
\scriptsize
\setlength{\tabcolsep}{3.5pt}
\renewcommand{\arraystretch}{1.12}

\begin{tabularx}{\linewidth}{
    @{}
    >{\raggedright\arraybackslash}p{0.18\linewidth}
    >{\raggedright\arraybackslash}p{0.38\linewidth}
    >{\raggedright\arraybackslash}X
    @{}
}
\toprule
\textbf{Method}
& \textbf{\#Tok.}
& \textbf{Accounting Note} \\
\midrule

Full Text / GRPO / OPSD
& $P+A$
& $P$ contains the full visible history, query, and instructions. \\

xRAG
& $T_{\mathrm{retriever}}+P+K_{\mathrm{xRAG}}+A$
& Counts retriever inputs and projected retrieval positions consumed by the reader. \\

AutoCompressor
& $\sum_s T_{\mathrm{AC}}^{(s)}+P+K_{\mathrm{AC}}+A$
& Each compression pass includes processed text and carried summary positions. \\

MemGen
& $T_{\mathrm{Weaver}}+K_{\mathrm{Weaver}}
   +T_{\mathrm{reasoner}}+K_{\mathrm{reasoner}}+A$
& Weaver and reasoner computation are counted separately; repeated history reads are retained. \\

\textbf{\texttt{MemFold}}
& $T_{\mathrm{writer\text{-}in}}
  +T_{\mathrm{memory\text{-}out}}
  +T_{\mathrm{compressor\text{-}in}}
  +P+K+A$
& Counts memory generation and subsequent compressor encoding as separate operations. \\

\bottomrule
\end{tabularx}
\end{table}

\section{Theoretical Analysis of \texorpdfstring{$\mathcal{L}_{\mathrm{OPD}}$}{L\_OPD}}
\label{app:theory}

We analyze the sampled-token surrogate introduced in
Section~\ref{sec:on_policy_optimization}.
Let
$\Delta_{i,t}=\ell_{T,i,t}-\ell_{\theta,i,t}$
and
$\bar g_{i,t}=\operatorname{sg}[\sigma(\beta\Delta_{i,t})]$,
where $\beta>0$.
During differentiation, sampled trajectories and their masks are held fixed.
We distinguish the gradient computed on a sampled batch from its idealized
conditional expectation, and state explicitly which properties hold for each.

\subsection{Gradient of the Sampled Surrogate}

\begin{proposition}[Gradient of the sampled surrogate]
\label{prop:opd_grad}
Let $\widehat{\mathbb{E}}_{\mathcal B}$ denote the masked token average over a
fixed rollout batch $\mathcal B$. The implemented surrogate is
\begin{equation}
\widehat{\mathcal L}_{\mathrm{OPD}}
=
\widehat{\mathbb{E}}_{\mathcal B}
\left[
\bar g_{i,t}
\left(\operatorname{sg}[\ell_{T,i,t}]-\ell_{\theta,i,t}\right)
\right],
\end{equation}
and its gradient is
\begin{equation}
\nabla_\theta\widehat{\mathcal L}_{\mathrm{OPD}}
=
-\widehat{\mathbb{E}}_{\mathcal B}
\left[\bar g_{i,t}\,\nabla_\theta\ell_{\theta,i,t}\right].
\label{eq:opd_grad}
\end{equation}
\end{proposition}

\begin{proof}
The sampled tokens, masks, gate, and teacher log-probabilities are constant in
the backward pass. Differentiating each summand therefore leaves only the
student log-probability term.
\end{proof}

The update is thus gradient-equivalent to gate-weighted negative
log-likelihood on student-generated tokens, with no gradient through the
teacher or the gate. This is a surrogate gradient with fixed samples, rather
than a total derivative through the rollout distribution.

\paragraph{Gate semantics.}
The gate is monotone in the sampled-token confidence gap:
\begin{equation}
\bar g_{i,t}
\begin{cases}
\to 1, & \beta\Delta_{i,t}\to+\infty,\\
=1/2, & \Delta_{i,t}=0,\\
\to 0, & \beta\Delta_{i,t}\to-\infty.
\end{cases}
\end{equation}
Tokens assigned substantially higher probability by the teacher receive larger
reinforcement weights, equal log-probabilities give weight $1/2$, and tokens
assigned substantially higher probability by the student receive weights
approaching zero. These regimes describe relative confidence on the sampled
token; they do not by themselves establish token correctness or identify the
information lost during compression.

\subsection{Relationship to Reverse KL}
\label{sec:Relation_KL}

Fix a query and generated prefix, denoted collectively by $h$, and write
\[
p_\theta(a)=\pi_\theta(a\mid h,Z),
\qquad
q(a)=\pi_T(a\mid h,M),
\qquad
\Delta(a)=\log q(a)-\log p_\theta(a).
\]
Assume a finite vocabulary and strictly positive probabilities.
With $q$ fixed, the reverse KL
$D_{\mathrm{KL}}(p_\theta\Vert q)
=\mathbb{E}_{a\sim p_\theta}[\log p_\theta(a)-\log q(a)]$
has gradient
\begin{equation}
\nabla_\theta D_{\mathrm{KL}}(p_\theta\Vert q)
=
-\mathbb{E}_{a\sim p_\theta}
\left[\Delta(a)\,\nabla_\theta\log p_\theta(a)\right],
\label{eq:rkl_grad}
\end{equation}
where the term
$\mathbb{E}_{a\sim p_\theta}[\nabla_\theta\log p_\theta(a)]$
arising from differentiating $\log p_\theta$ vanishes by the score-function
identity. Both objectives can therefore be estimated from student-generated
samples. At the level of individual samples, they differ in the coefficient
applied to $\nabla_\theta\log p_\theta(a)$: reverse KL uses the unbounded,
signed log-ratio $\Delta(a)$, whereas OPD uses the bounded, nonnegative weight
$\bar g(a)=\sigma(\beta\Delta(a))\in(0,1)$.
Note, however, that nonnegative sample weights do not imply that every token
probability increases: probability normalization and shared parameters couple
the updates across tokens. The following results make the relationship precise.

\subsection{Expected Update: Stationarity and Bounded Influence}

\begin{proposition}[Conditional stationarity and local attenuation]
\label{prop:opd_stationary}
At a fixed prefix $h$, suppose tokens are sampled exactly from the same
distribution $p_\theta$ used to compute their log-probabilities, and define the
expected surrogate gradient
\begin{equation}
\mathcal G(\theta;h)
=
-\mathbb{E}_{a\sim p_\theta}
\left[\bar g(a)\,\nabla_\theta\log p_\theta(a)\right].
\end{equation}
Then $p_\theta=q$ implies $\mathcal G(\theta;h)=0$. More generally,
\begin{equation}
\|\mathcal G(\theta;h)\|
\le
\frac{\beta}{4}
\sqrt{\mathbb{E}_{a\sim p_\theta}[\Delta(a)^2]}\,
\sqrt{\mathbb{E}_{a\sim p_\theta}
\left[\|\nabla_\theta\log p_\theta(a)\|^2\right]}.
\label{eq:opd_attenuation_bound}
\end{equation}
In particular, if the score second moment is bounded,
$\mathbb{E}_{a\sim p_\theta}[\|\nabla_\theta\log p_\theta(a)\|^2]\le S^2$,
then $\|\mathcal G(\theta;h)\|\le\frac{\beta S}{4}
\sqrt{\mathbb{E}_{a\sim p_\theta}[\Delta(a)^2]}$, which vanishes as the
mean-squared log-probability gap vanishes.
\end{proposition}

\begin{proof}
By the score-function identity,
$\mathbb{E}_{a\sim p_\theta}[\nabla_\theta\log p_\theta(a)]
=\sum_a\nabla_\theta p_\theta(a)=\nabla_\theta 1=0$.
Subtracting the constant $\tfrac12$ from the gate therefore leaves the
expectation unchanged:
\begin{equation}
\mathcal G(\theta;h)
=
-\mathbb{E}_{a\sim p_\theta}
\left[\left(\bar g(a)-\tfrac12\right)\nabla_\theta\log p_\theta(a)\right].
\label{eq:opd_centered}
\end{equation}
If $p_\theta=q$, then $\Delta(a)=0$ and $\bar g(a)=\tfrac12$ for every token,
proving stationarity. Since $\sigma'(u)\le\tfrac14$ and $\sigma(0)=\tfrac12$,
$|\bar g(a)-\tfrac12|\le\tfrac{\beta}{4}|\Delta(a)|$.
Applying the triangle inequality and the Cauchy--Schwarz inequality to
Eq.~(\ref{eq:opd_centered}) yields
Eq.~(\ref{eq:opd_attenuation_bound}).
\end{proof}

\begin{corollary}[Expected OPD as a saturated reverse KL]
\label{cor:opd_rkl}
Under the assumptions of Proposition~\ref{prop:opd_stationary},
\begin{equation}
\mathcal G(\theta;h)
=
-\frac12\,
\mathbb{E}_{a\sim p_\theta}
\left[
\tanh\!\left(\tfrac{\beta\Delta(a)}{2}\right)
\nabla_\theta\log p_\theta(a)
\right],
\label{eq:opd_tanh}
\end{equation}
and
\begin{equation}
\Big\|
\mathcal G(\theta;h)
-\tfrac{\beta}{4}\,\nabla_\theta D_{\mathrm{KL}}(p_\theta\Vert q)
\Big\|
\le
\frac{\beta^3}{48}\,
\mathbb{E}_{a\sim p_\theta}
\left[|\Delta(a)|^3\,\|\nabla_\theta\log p_\theta(a)\|\right].
\label{eq:opd_rkl_approx}
\end{equation}
\end{corollary}

\begin{proof}
Eq.~(\ref{eq:opd_tanh}) follows from Eq.~(\ref{eq:opd_centered}) and
the identity $\sigma(u)-\tfrac12=\tfrac12\tanh(u/2)$.
By Eq.~(\ref{eq:rkl_grad}),
\[
\mathcal G(\theta;h)-\tfrac{\beta}{4}\nabla_\theta D_{\mathrm{KL}}(p_\theta\Vert q)
=
-\mathbb{E}_{a\sim p_\theta}
\left[
\left(
\tfrac12\tanh\!\left(\tfrac{\beta\Delta(a)}{2}\right)-\tfrac{\beta\Delta(a)}{4}
\right)
\nabla_\theta\log p_\theta(a)
\right].
\]
Since $|\tanh(u)-u|\le|u|^3/3$, the scalar coefficient is bounded in absolute
value by $\tfrac12\cdot\tfrac13\left|\tfrac{\beta\Delta(a)}{2}\right|^3
=\tfrac{\beta^3}{48}|\Delta(a)|^3$. The triangle inequality completes the
proof.
\end{proof}

Corollary~\ref{cor:opd_rkl} clarifies what the gate does in expectation.
The expected OPD update replaces the unbounded log-ratio $\Delta(a)$ in the
reverse-KL gradient with the bounded, sign-preserving coefficient
$\tfrac12\tanh(\beta\Delta(a)/2)\in(-\tfrac12,\tfrac12)$.
When the confidence gap is small, OPD behaves as a reverse KL toward the
textual-memory teacher scaled by $\beta/4$. When the gap is large, the
influence of any single token is capped, so tokens on which the frozen teacher
is strongly over- or under-confident cannot dominate the update. This is
consistent with treating $\pi_T$ as a reference reader of the uncompressed
memory rather than as an accuracy oracle. The nonnegativity of $\bar g$ is a
property of the per-sample weights: once the zero-mean baseline $\tfrac12$ is
removed, the effective expected coefficient is signed, and the expected update
does move the student toward the teacher, with bounded strength.

\begin{proposition}[OPD does not import the teacher's own choices]
\label{prop:opd_support}
Fix a prefix $h$ and treat the logits $z\in\mathbb{R}^{|\mathcal V|}$ of
$p_\theta(\cdot\mid h,Z)=\operatorname{softmax}(z)$ as free parameters.
For tokens $a_1,\dots,a_n$ sampled at $h$, the descent direction of
$\widehat{\mathcal L}_{\mathrm{OPD}}$ with respect to $z$ is
\begin{equation}
-\nabla_z\widehat{\mathcal L}_{\mathrm{OPD}}
=\frac1n\sum_{j=1}^{n}\bar g(a_j)\,\big(e_{a_j}-p_\theta\big),
\end{equation}
where $e_a$ is the one-hot vector of token $a$. Consequently:
(i) the logit of every token outside $\{a_j\}_{j=1}^n$ decreases or stays
unchanged, regardless of the teacher $q$; and
(ii) under exact sampling from $p_\theta$, the expected change in the logit of
any token $b$ is
\begin{equation}
p_\theta(b)\Big(\bar g(b)-\mathbb{E}_{a\sim p_\theta}[\bar g(a)]\Big),
\end{equation}
whose magnitude is at most $p_\theta(b)$.
\end{proposition}

\begin{proof}
For a softmax policy, $\nabla_z\log p_\theta(a)=e_a-p_\theta$, which gives the
descent direction. For $b\notin\{a_j\}$, its $b$-th component is
$-\frac1n\sum_j\bar g(a_j)\,p_\theta(b)\le 0$, proving (i). Taking the
expectation over $a\sim p_\theta$, the $b$-th component becomes
$p_\theta(b)\bar g(b)-p_\theta(b)\,\mathbb{E}_{a\sim p_\theta}[\bar g(a)]$.
Since $\bar g\in(0,1)$, the factor in parentheses lies in $(-1,1)$, proving (ii).
\end{proof}

Proposition~\ref{prop:opd_support} separates two senses in which a student can
be ``pulled toward'' a teacher. Imitating the teacher's own choices, as in
supervised fine-tuning on teacher outputs or forward-KL distillation, follows
$-\nabla_z\,\mathrm{CE}(q,p_\theta)=q-p_\theta$. Its $b$-th component
$q(b)-p_\theta(b)$ can be large even when $p_\theta(b)\approx 0$, so tokens the
teacher prefers are promoted whether or not the student produces them. Under OPD, by contrast, tokens the student did not sample are never promoted,
and in expectation every logit change is scaled by the student's own
probability. The pull established in Corollary~\ref{cor:opd_rkl} is therefore confined to re-ranking the student's own candidates according to the teacher's relative confidence; the teacher's own choices outside the student's support are never imported. The statement holds for the logits at a fixed prefix; with shared parameters $\theta$, updates are additionally coupled across prefixes, as noted in Section~\ref{sec:Relation_KL}.

\subsection{Remarks on the Implemented Estimator}

\begin{remark}
Decomposing Eq.~(\ref{eq:opd_grad}) as
\[
\nabla_\theta\widehat{\mathcal L}_{\mathrm{OPD}}
=
-\widehat{\mathbb{E}}_{\mathcal B}
\left[\left(\bar g_{i,t}-\tfrac12\right)\nabla_\theta\ell_{\theta,i,t}\right]
-\tfrac12\,\widehat{\mathbb{E}}_{\mathcal B}
\left[\nabla_\theta\ell_{\theta,i,t}\right]
\]
separates the saturated reverse-KL term analyzed above from a baseline term.
Under exact sampling, the baseline term has zero mean and contributes only
variance. Under tempered or nucleus sampling, however, it becomes a
self-imitation term toward the truncated sampling distribution, which tends to
sharpen $p_\theta$. Our PersonaMem rollouts use temperature $1.0$ and
top-$p=0.98$, close to exact sampling. The LoCoMo rollouts use temperature
$0.8$ and top-$p=0.95$, where this effect is larger. Sample-dependent
per-sequence normalization $1/|y_i|$ further prevents exact cancellation.
\end{remark}

\begin{remark}
The results above are local, conditional statements about the expected update
at a fixed prefix. They do not guarantee that training reaches the regime
$p_\theta\approx q$, that attenuation is monotone over training, or that the
OPD coefficient $\lambda_{\mathrm{OPD}}$ need not be tuned.
\end{remark}

\section{Limitations}
\label{app:limitations}
First, the memory writer is initialized on memories extracted by a stronger
external model; although this model is not needed at inference, the quality of
the initial textual memory depends on it, and learning the writer without such
supervision is left to future work. Second, our experiments cover Qwen
backbones of up to 7B parameters; whether the same gains hold for larger
models and other model families remains to be verified.





\end{document}